\documentclass[twoside,11pt]{article}

\PassOptionsToPackage{numbers, compress}{natbib}

\usepackage{jmlr2e}

\usepackage[utf8]{inputenc} 
\usepackage{url}            
\usepackage{booktabs}       
\usepackage{amsfonts}       
\usepackage{nicefrac}       
\usepackage{microtype}      

\usepackage{graphicx}

\usepackage{amssymb,bm,amsmath,amsfonts,subfigure,amsthm,tikz}
\usepackage{nicefrac}       
\usepackage[normalem]{ulem}
\usepackage{color}

\usepackage{float}
\usepackage{comment}
\usepackage{multicol,multirow}
\usepackage{bbm}
\usepackage{caption}

\usepackage{hyperref}       

\newtheorem{claim}{Claim}[section]
\newtheorem{thm}{Theorem}[section]

\def\eq#1{(\ref{#1})}

\def\beginmat{ \left( \begin{array} }
\def\endmat{ \end{array} \right) }

\def\log{{\rm log}}
\def\tr{{\rm tr}}
\def\cond{\, | \,}
\newcommand{\bbeta}{{\bm{\beta}}}

\newcommand{\tbbeta}{{\widetilde{\bm{\beta}}}}

\newcommand{\bb}{\mathbf{b}}

\newcommand{\bx}{\mathbf{x}}
\newcommand{\by}{\mathbf{y}}

\newcommand{\bw}{\mathbf{w}}

\newcommand{\bK}{\mathbf{K}}
\newcommand{\bV}{\mathbf{V}}

\newcommand{\bX}{\mathbf{X}}

\newcommand{\bH}{\mathbf{H}}

\newcommand{\bI}{\mathbf{I}}

\newcommand{\V}{\mathbb{V}}
\newcommand{\N}{{\cal N}}

\newcommand{\bvarepsilon}{\boldsymbol\varepsilon}

\newcommand{\tbeta}{{\widetilde{\beta}}}

\newcommand{\btheta}{\boldsymbol\theta}

\newcommand{\bphi}{\boldsymbol\phi}

\newcommand{\balpha}{\boldsymbol\alpha}
\newcommand{\bmu}{\boldsymbol\mu}

\newcommand{\bDelta}{\boldsymbol\Delta}

\newcommand{\bOmega}{\boldsymbol\Omega}

\newcommand{\bLambda}{\boldsymbol\Lambda}

\newcommand*\diff{\mathop{}\!\mathrm{d}}

\def\f{{\bm f}} 
\def\bX{{\mathbf X}}

\newcommand{\X}{{\bf{X}}}

\newcommand{\x}{{\bf{x}}}
\newcommand{\y}{{\bf{y}}}

\newcommand{\cN}{{\cal N}}

\newcommand{\T}{\intercal}

\jmlrheading{21}{2020}{1-30}{04/20}{--/--}{---}{Jonathan Ish-Horowicz, Dana Udwin, Kayla Scharfstein, Seth Flaxman, Lorin Crawford and Sarah Filippi}

\ShortHeadings{Interpreting Deep Neural Networks Through Variable Importance}{Ish-Horowicz et al.}
\firstpageno{1}

\begin{document}

\title{Interpreting Deep Neural Networks Through Variable Importance}

\author{\name Jonathan Ish-Horowicz \email jonathan.ish-horowicz17@imperial.ac.uk \\
       \addr Department of Mathematics, Imperial College London\\
       London SW7 2AZ, UK
       \AND
       \name Dana Udwin \email dana\_udwin@brown.edu \\
       \addr Department of Biostatistics, Brown University\\
       Providence, RI, USA
       \AND
       \name Kayla Scharfstein \email kayla\_scharfstein@brown.edu \\
       \addr Division of Applied Mathematics, Brown University\\
       Providence, RI, USA
       \AND
       \name Seth Flaxman \email s.flaxman@imperial.ac.uk\\
       \addr Department of Mathematics, Imperial College London\\
       London SW7 2AZ, UK
       \AND
       \name Lorin Crawford \email lorin\_crawford@brown.edu\\
       \addr Department of Biostatistics, Brown University\\
       Providence, RI, USA
       \AND
       \name Sarah Filippi \email s.filippi@imperial.ac.uk\\
       \addr Department of Mathematics, Imperial College London\\
       London SW7 2AZ, UK}

\editor{Editor names}
\maketitle


\begin{abstract}
    While the success of deep neural networks is well-established across a variety of domains, our ability to explain and interpret these methods is limited. Unlike previously proposed \textit{local} methods which try to explain particular classification decisions, we focus on \textit{global} interpretability and ask a generally applicable, yet understudied, question: given a trained model, which input features are the most important? In the context of neural networks, a feature is rarely important on its own, so our strategy is specifically designed to leverage partial covariance structures and incorporate variable interactions into our proposed feature ranking. Here, we extend the recently proposed ``RelATive cEntrality'' (RATE) measure \citep{crawford2018variable} to the Bayesian deep learning setting. Given a trained network, RATE applies an information theoretic criterion to the posterior distribution of effect sizes to assess feature significance. Importantly, unlike competing approaches, our method does not require tuning parameters which can be costly and difficult to select. We demonstrate the utility of our framework on both simulated and real data.
\end{abstract}



\section{Introduction} \label{sec:introduction}

Due to their high predictive performance, deep neural networks (DNNs) have become increasingly ubiquitous in many fields including computer vision and natural language processing \citep{lecun2015deep}. Unfortunately, DNNs operate as ``black boxes'': users are rarely able to understand the internal workings of the network. As a result, DNNs have not been widely adopted in scientific settings, where variable selection tasks are often as important as prediction --- one particular example being the identification of biomarkers related to the progression of a disease. While DNNs are beginning to be used in high-risk decision-making fields (e.g., automated medical diagnostics or self-driving cars \citep{lundervold2019overview}), it is critically important that methods do not make predictions based on artefacts or biases in the training data. Therefore, there is both a strong theoretical and practical motivation to increase the global interpretability of DNNs, and to better characterize the types of relationships upon which they rely.

The increasingly important concept of interpretability in machine learning still lacks a well-established definition in the literature. Despite recent surveys \citep{guidotti2018survey,carvalho2019machine} and proposed guidelines \citep{hall2019guidelines} to address this issue, conflicting views on how interpretability should be evaluated still remain. Variable importance is one possible approach to achieve global interpretability, where the goal is to rank each input feature based on its contributions to predictive accuracy. This is in contrast to local interpretability, which aims to simply provide an explanation behind a specific prediction or group of predictions \citep{arya2019one,clough2019global}. In this paper, we follow a definition which refers to interpretability as ``the ability to explain or to present in understandable terms to a human'' \citep{doshi2017towards}. To this end, our main contribution is focused on global interpretability: we address the problem of identifying important predictor variables given a trained neural network, focusing especially on settings in which variables (or groups of variables) are intrinsically meaningful.

Here, we describe an approach to interpret deep neural networks using ``RelATive cEntrality'' (RATE) \citep{crawford2018variable}, a recently-proposed variable importance criterion for Bayesian models. This flexible approach can be used with any network architecture where some notion of uncertainty can be computed over the predictions. The rest of the paper is structured as follows. Section \ref{sec:recent-work} outlines related work on the interpretation of deep neural networks. Section \ref{sec:background} describes the RATE computation within the context for which it was originally proposed (Gaussian process regression). Section \ref{sec:contributions} contains the main methodological innovations of this paper. Here, we present a unified framework under which RATE can be applied to deep neural networks based on variational Bayes. In Section \ref{sec:results}, we demonstrate the utility of our method in various simulation scenarios and real data applications, and compare to competing approaches. Section \ref{sec:grouprate} describes an extension of RATE to calculate the importance of groups of variables (groupRATE) and uses it to assess gene importance in genome-wide association studies.



\section{Related Work} \label{sec:recent-work}

In the absence of a robustly defined metric for interpretability, most work on DNNs has centered around locally interpretable methods with the goal to explain specific classification decisions with respect to input features \citep{bach2015pixel,ribeiro2016should,shrikumar2016not,ancona2017towards,sundararajan2017axiomatic,adebayo2018sanity}. In this work, we focus instead on global interpretability where the goal is to identify predictor variables that best explain the overall performance of a trained model. Previous work in this context have attempted to solve this issue by selecting inputs that maximize the activation of each layer within the network \citep{erhan2009visualizing}. Another viable approach for achieving global interpretability is to train more conventional statistical methods to mimic the predictive behavior of a DNN. This ``student'', or or ``mimic'' model is then retrospectively used to explain the predictions that a DNN would make at a global level (contrasting with . Such mimic models are typically trained on the soft labels (the predicted probabilities) output by the network, as these are often more informative than the corresponding hard (class) labels \citep{ba2014deep,hinton2015distilling,che2016interpretable}.

For example, using a decision tree \citep{frosst2017distilling,kuttichira2019explaining} or falling rule list \citep{wang2015falling} can yield straightforward characterizations of predictive outcomes. Unfortunately, these simple models can struggle to mimic the accuracy of DNNs effectively. A random forest (RF) or gradient boosting machine (GBM), on the other hand, is much more capable of matching the predictive power of DNNs. Measures of feature importance can be computed for RFs and GBMs by permuting information within the input variables and examining this null effect on test accuracy, or by calculating Mean Decrease Impurity (MDI) \citep{breiman2001random}. The ability to establish variable importance in random forests is a significant reason for their popularity in fields such as the life and clinical sciences \citep{chen2007forest}, where random forest and gradient boosting machine mimic models have been used as interpretable predictive models for patient outcomes \citep{che2016interpretable}. A notable drawback of RFs and GBMs is that it can take a significant amount of training time to achieve accuracy comparable to the DNNs that they serve to mimic. This provides motivation for our direct approach, avoiding the need to train a separate model.


\section{Relevant Background} \label{sec:background}

In this section, we give a brief review on previous results that are relevant to our main methodological innovations. Throughout, we assume access to some trained Bayesian model, with the ability to draw samples from its posterior predictive distribution. This reflects the \textit{post-hoc} nature of our objective of finding important subsets of variables.

\subsection{Effect Size Analogues for Kernel Regression Models}

Assume that we have an $n$-dimensional response vector $\y$ and an $n\times p$ design matrix $\bX$ with $p$ covariates. To begin, we consider a standard linear regression model where 
\begin{align}
\y = \f + \bvarepsilon, \quad \quad \f = \X\bbeta, \quad \quad \bvarepsilon\sim\N(\bm{0},\tau^2\bI) \,,
\end{align}
where $\bbeta$ is a $p$-dimensional vector of additive effect sizes, $\bvarepsilon$ is an $n$-dimensional vector of error terms that are assumed to follow a multivariate normal distribution with mean zero and scaled variance term $\tau^2$, and $\bI$ is an identity matrix. In classical statistics, a least squares estimate of the regression coefficients is defined as the projection of the response variable onto the column space of the data: $\mbox{Proj}({\mathbf X},\y) = \X^{\dagger} {\mathbf y}$, with $\X^{\dagger}$ being the Moore-Penrose pseudo-inverse. In the Bayesian nonparametric setting, we relax the additive assumption in the covariates and consider a learned nonlinear function $\f$ that has been evaluated on the $n$-observed samples \citep{Kolmogorov:1960aa,Scholkopf:2001aa,scholkopf2002learning}
\begin{align}
\y = \f + \bvarepsilon, \quad \quad \f \sim\N(\bm{0},\bK), \quad \quad \bvarepsilon\sim\N(\bm{0},\tau^2\bI) \,,
\end{align}
where, in addition to previous notation and without loss of generality, $\f = [f_1(\x_1),\ldots,f_n(\x_n)]^{\T}$ is assumed to come from a multivariate normal with mean $\bm{0}$ and covariance (or kernel) matrix $\bK$. The above probabilistic model is commonly referred to as a ``weight-space'' view on Gaussian processes \citep{rasmussen2006gaussian}. Generally, this class of models posit that $\f$ lives within a reproducing kernel Hilbert space (RKHS) defined by some nonlinear covariance function $k_{rs} = k(\x_r,\x_s)$ for each element in $\bK$. Many of these covariance functions have been shown to implicitly account for higher-order interactions between features, which often lead to more accurate predictions for complex data types (e.g., the radial basis function) \citep{Cotter:2011aa,Jiang:2015aa}. 

The effect size analogue, denoted $\tbbeta$, represents the nonparametric equivalent to coefficient estimates in linear regression using generalized ordinary least squares. This can then be defined as the result of projecting the learned nonlinear vector $\bm{f}$ onto the original design matrix $\X$,  
\begin{align}
    \tbbeta = \mbox{Proj}({\mathbf X},\bm{f}). \label{eq:geff}
\end{align}
Some intuition can be gained as follows. After having fit a probabilistic model, we consider the fitted values $\f$ and regress these predictions onto the input variables so as to see how much variance these features explain. This is a simple way of understanding the relationships that the model has learned. The coefficients produced by this linear projection have their normal interpretation: they provide a summary of the relationship between the covariates in $\X$ and $\f$. For example, while holding everything else constant, increasing some feature $\x_j$ by $1$ will increase $\f$ by $\widetilde{\beta}_j$. In the case of kernel machines, theoretical results for identifiability and sparsity conditions of the effect size analogue have been previously developed when using the Moore-Penrose pseudo-inverse as the projection operator \citep{crawford2018bayesian}.  

\subsection{Variable Importance using Relative Centrality Measures}

Similar to regression coefficients in linear models, effect size analogues are not used to solely determine variable significance. Indeed, there are many approaches to infer associations based on the magnitude of effect size estimates, but many of these techniques rely on arbitrary thresholding and fail to account for key covarying relationships that exist within the data. The ``RelATive cEntrality'' measure (or RATE) was developed as a \textit{post-hoc} approach for variable selection that mitigates these concerns \citep{crawford2018variable}.

Consider a sample from the predictive distribution of $\tbbeta$, obtained by iteratively transforming draws from the posterior of $\f$ via the deterministic projection specified in Equation \eq{eq:geff}. The RATE criterion summarizes how much any one variable contributes to the total information the model has learned. Effectively, this is done by taking the Kullback-Leibler divergence (KLD) between \textit{(i)} the conditional posterior predictive distribution $p(\widetilde{\bm{\beta}}_{-j}\cond \widetilde{\beta}_j=0)$ with the effect of the $j$-th predictor being set to zero, and \textit{(ii)} the marginal distribution $p(\widetilde{\bm{\beta}}_{-j})$ with the effects of the $j$-th predictor being integrated out. In this work, we denote the RATE criterion as
$$\gamma_j= \dfrac{\text{KLD}_j}{\sum_{k}\text{KLD}_k}\,,$$
where $\gamma_j$ quantifies the importance of the $j$ variable in the model and
\begin{align}
    \text{KLD}_j &:= \text{KL}\left(p(\widetilde{\bm{\beta}}_{-j})\,\|\, 
    p(\widetilde{\bm{\beta}}_{-j}\cond \widetilde{\beta}_j=0)\right) = \int\log\left(\frac{p(\widetilde{\bm{\beta}}_{-j})}{p(\widetilde{\bm{\beta}}_{-j}\cond \widetilde{\beta}_j=0)}\right) p(\widetilde{\bm{\beta}}_{-j}) \diff\widetilde{\bm{\beta}}_{-j}.
    \label{eq:KLD}
\end{align}
Note that the $\text{KLD}_j$ is a non-negative quantity, and equals zero if and only if the $j$-th variable is of little importance, since removing its effect has no influence on the other variables. In addition, the RATE criterion is bounded within the range $\gamma_j\in[0,1]$ and has the natural interpretation of measuring a variable's relative entropy --- with a higher value equating to more importance. 


\section{Methodological Contributions} \label{sec:contributions}

We now detail the main methodological contributions of this paper. First, we describe our motivating deep neural network framework, which is based on variational Bayesian inference. Next, we propose a new effect size analogue projection that is more robust to collinear input data. Lastly, we derive closed-form solutions for the RATE methodology under this new framework. While this framework satisfies the requirements for RATE, we can also calculate RATE values for any neural network for which we can sample from the posterior predictive distribution. However, this framework provides one key advantage for RATE in that it permits these closed-form solutions.

\subsection{Motivating Neural Network Architecture} \label{subsec:bnn-arch-for-rate}

In order to make RATE amenable for deep learning, we are required to take a probabilistic view on prediction. This is possible by using a Bayesian neural network (BNN). In contrast to  a ``standard'' neural network, which uses maximum likelihood point-estimates for its parameters, a Bayesian neural network assumes a prior distribution over its weights. The posterior probability over the weights, learned during the training phase, can then be used to compute the posterior predictive distribution. Once again, we consider a general predictive task with an $n$-dimensional set of response variables and an $n\times p$ design matrix $\X$ with $p$ covariates. For this problem, we assume the following hierarchical network architecture to learn the predicted response for each observation in the data
\begin{align}
\widehat{\y} = \sigma(\f), \quad \f = \mathbf{H}(\boldsymbol{\theta})\textbf{w} + \textbf{b}, \quad \textbf{w} \sim \pi(\cdot) \, , \label{eq:logit-batch}
\end{align}
where $\sigma(\cdot)$ is a link function, $\btheta$ is a vector of inner layer weights, and $\bm{f}$ is an $n$-dimensional vector of smooth latent values or ``functions'' that need to be estimated. These function values are also known as logits. Here, we use $\bH(\btheta)$ to denote an $n\times l$ matrix of activations from the penultimate layer (which are fixed given a set of inputs $\bX$ and point estimates for the inner layer weights $\btheta$), $\bw\sim \pi$ is a $l$-dimensional vector of weights at the output layer assumed to follow prior distribution $\pi$, and $\bb$ is an $n$-dimensional vector of the deterministic bias that is produced during the training phase.

The hierarchical structure of Equation \eqref{eq:logit-batch} is motivated by the fact that we are most interested in the posterior distribution of the latent variables when computing the effect size analogues and, subsequently, interpretable RATE measures. To this end, we may logically split the network architecture into three key components: \textit{(i)} an input layer of the original predictor variables, \textit{(ii)} hidden layers where parameters are deterministically computed, and \textit{(iii)} the outer layer where the parameters and activations are treated as random variables. Since the resulting functions are a linear combination of these components, their joint distribution will be closed-form if the posterior distribution of the weight parameters can also be written in closed-form. By restricting that only the weights in the outer layer are random also brings computational benefits during network training as it drastically reduces the number of parameters (versus learning a posterior for every parameter in the network).

There are two important features that come with this neural network setup. First, we may easily generalize this type of architecture to different predictive tasks through the link function $\sigma(\cdot)$. For example, we may apply our model to the binary classification problem by increasing the number of output nodes to match the number of categories, and redefining link function to be the sigmoid function. Regression is even simpler: we let the link function be the identity. Second, the structure of the hidden layers can be of any size or type, provided that we have access to draws of the posterior predictive distribution for the response variables. Ultimately, this flexibility means that a wide range of existing probabilistic network architectures can be easily modified to be used with RATE. The simplest example of such an architecture is illustrated in Figure~\ref{fig:bnn-example-architecture}.

\begin{figure}[ht]
\centering
\includegraphics[width=0.64\textwidth]{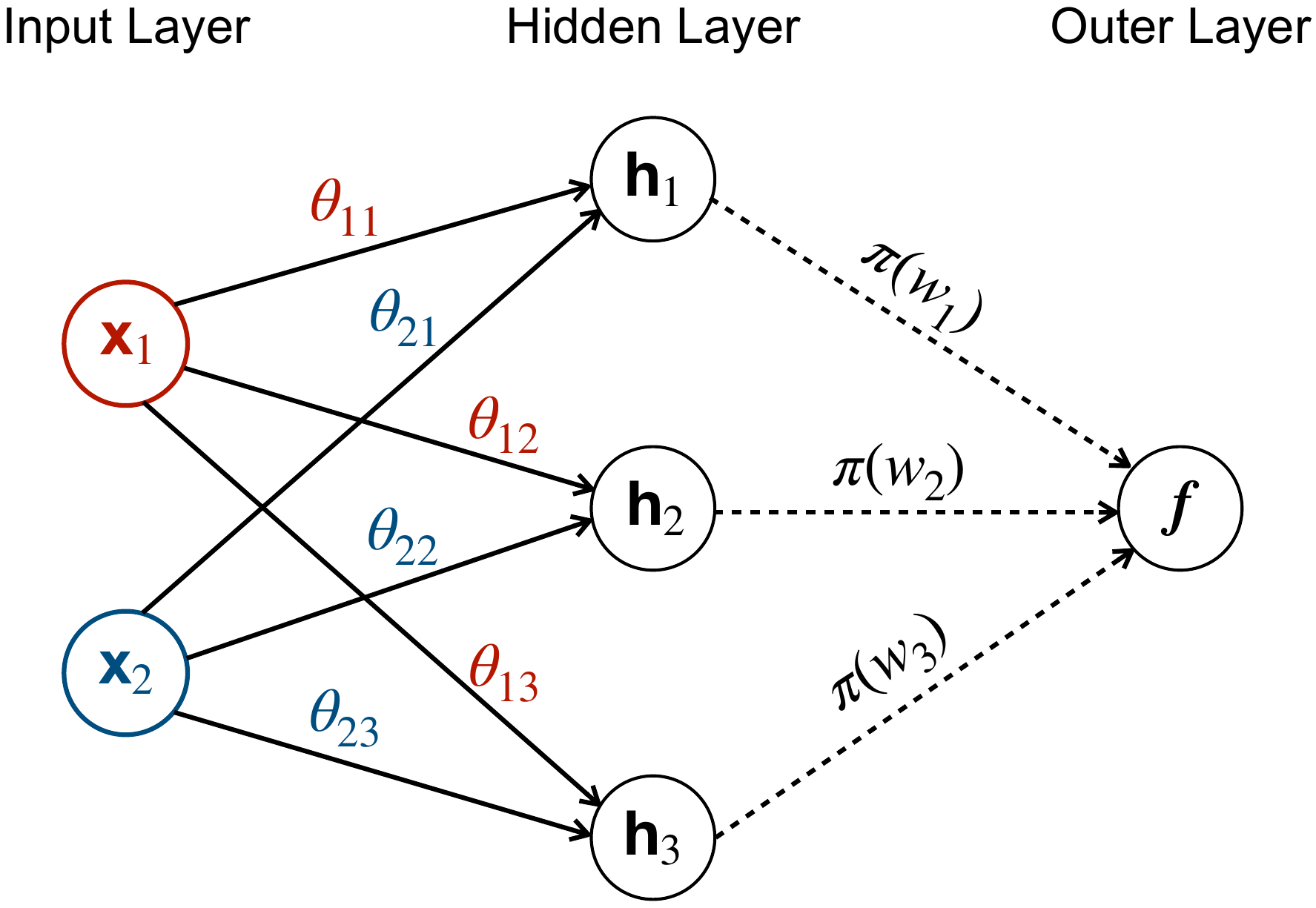}
	\caption{An example of the probabilistic neural network architecture used in this work. The first layer weights $\boldsymbol{\theta}$ are point estimates, while the outer layer weights $\mathbf{w}$ are assumed to be random variables following the prior distribution $\pi=(\pi(w_1),\pi(w_2),\pi(w_3))$. The input variables $\bf{x}_1$ and $\bf{x}_2$ are fed through the hidden layers $(\mathbf{h}_1,\mathbf{h}_2,\mathbf{h}_3)$ as weighted linear combinations of the $\btheta$. Estimates of the predicted functions $\f$ are obtained via a linear combination of the activations and samples from the posterior distribution of the outer layer weights $(w_1,w_2,w_3)$. Note that this figure does not include the deterministic bias terms used in Equation \eq{eq:logit-batch}.}
\label{fig:bnn-example-architecture}
\end{figure}

\subsection{Posterior Inference with Variational Bayes} \label{subsec:post-inf-var-bayes}

As the size of datasets in many application areas continues to grow, it has become less feasible to implement traditional Markov Chain Monte Carlo (MCMC) algorithms for inference. This has motivated approaches for supervised learning that are based on variational Bayes and the stochastic optimization of a variational lower bound \citep{hinton1993keeping,barber1998ensemble,graves2011practical}. In this work, we use variational Bayes because it has the additional benefit of providing closed-form expressions for the posterior distribution of the weights in the outer layer $\mathbf{w}$ --- and, subsequently, the functions $\f$. Here, we first specify a prior $\pi(\bw)$ over the weights and replace the intractable true posterior $p(\mathbf{w}\cond\y) \propto p(\y \cond \mathbf{w}) \pi(\mathbf{w})$ with an approximating family of distributions $q_{\boldsymbol{\phi}}(\mathbf{w})$. The variational parameters $\boldsymbol{\phi}$ are selected by minimizing the divergence $\mbox{KL}(q_{\boldsymbol{\phi}}(\mathbf{w})\,\|\,p(\mathbf{w}\cond\y))$, with respect to $\boldsymbol{\phi}$, with the goal of selecting the member of the approximating family that is closest to the true posterior. This is equivalent to maximizing the so-called variational lower bound. 

Since the architecture specified in Equation \eq{eq:logit-batch} contains point estimates at the hidden layers, we cannot train the network by simply maximizing the lower bound with respect to the variational parameters. Instead, all parameters must be optimized jointly as follows:
\begin{equation}
\underset{\boldsymbol{\phi},\btheta}{\arg\max} -\text{KL} (q_{\boldsymbol{\phi}}(\mathbf{w}) \, \| \, \pi(\mathbf{w})) +\mathbb{E}_{q_{\boldsymbol{\phi}}(\mathbf{w})} \left[\log \, p(\y\cond\mathbf{w}, \boldsymbol{\theta}) \right].
\end{equation}
We then use stochastic optimization to train the network. Depending on the chosen variational family, the gradients of the minimized $\text{KL} (q_{\boldsymbol{\phi}}(\mathbf{w}) \, \| \, \pi(\mathbf{w}))$ may be available in closed-form, while gradients of the log-likelihood $\log\, p(\y\cond\mathbf{w}, \boldsymbol{\theta})$ are evaluated using Monte Carlo samples and the local reparameterization trick \citep{kingma2015variational}. Following this procedure, we obtain an optimal set of parameters for $q_{\boldsymbol{\phi}}(\mathbf{w})$, with which we can sample posterior draws for the outer layer.

In this work, we choose independent Gaussians as the family of approximating distributions
\begin{align}
q_{\bphi}(\mathbf{w}) = \cN(\mathbf{m},\bV), \quad \bphi = \{\mathbf{m},\mathbf{V}\}, \label{eq:fac-q-phi}
\end{align}
with mean vector $\mathbf{m}$ and a diagonal covariance matrix $\bV$. This makes the mean-field assumption that the variational posterior fully factorizes over the elements of $\mathbf{w}$ \citep{blei2017variational}. One advantage of this choice is that it ensures that the predicted functions $\f$ will follow a multivariate Gaussian distribution as well. Using Equations \eqref{eq:logit-batch} and \eqref{eq:fac-q-phi}, we may derive the implied distribution over the latent values using the affine transformation property
\begin{align}
\f\cond\X,\y \sim \mathcal{N}(\mathbf{H}(\btheta)\mathbf{m} + \bb, \bH(\btheta)\bV\bH(\btheta)^{\T}).\label{eq:logit-posterior}
\end{align}
While the elements of $\mathbf{w}$ are independent, dependencies in the input data (via the hidden activations $\bH(\btheta)$) induce a non-diagonal covariance between the elements of $\f$.

\subsection{Effect Size Analogue for Bayesian Neural Networks}

After having conducted (variational) Bayesian inference, we now have the posterior $p(\f\cond\X,\y)$ in closed-form (Equation \eq{eq:logit-posterior}), which we can use to define an effect size analogue for neural networks. We could use the Moore-Penrose pseudo-inverse as proposed in \citep{crawford2018variable} but, in the case of highly correlated inputs, this operator suffers from instability (see a small simulation study in Supplementary Material), explaining the well-known phenomenon of linear regression  suffering in the presence of collinearity. While regularization poses a viable solution to this problem, the selection of an optimal penalty parameter is not always a straightforward task. As a result, we propose a much simpler projection operator that is particularly effective in application areas where data measurements can be perfectly collinear (e.g., pixels in an image). Our solution is to use a linear measure of dependence separately for each predictor based on the sample covariance. Namely, for each of the $p$ input variables
\begin{align}
\tbbeta = \mbox{Proj}(\bX,\bm{f}) = \text{cov}(\bX,\bm{f}) \,, \label{ESA}
\end{align}
where $\text{cov}(\bX,\bm{f}) = [\text{cov}(\x_1,\f),\ldots,\text{cov}(\x_p,\f)]$. Since it is based on the sample covariance, the effect size analogue has the form $\tbbeta = \X^{\T}\textbf{C}\f/(n-1)$ --- where $\textbf{C} = \bI - \bm{1}\bm{1}^{\T}/n$ denotes a centering matrix, $\bI$ is an $n$-dimensional identity matrix, and $\bm{1}$ is an $n$-dimensional vector of ones. Probabilistically, since the posterior of the function values $\f$ is normally distributed according to Equation \eqref{eq:logit-posterior}, the above is equivalent to assuming that $\tbbeta\cond\X,\y\sim\cN(\bmu,\bOmega)$ where 
\begin{align}
\boldsymbol{\mu} &=\frac{1}{n-1} \mathbf{X}^{\T}\textbf{C} \mathbf{H}(\btheta)\mathbf{m}, \quad \quad \boldsymbol{\Omega} =\frac{1}{(n-1)^2}\mathbf{X}^{\T}\textbf{C} \bH(\btheta)\bV\bH(\btheta)^{\T}\textbf{C}^{\T}\mathbf{X}.
\end{align}
Intuitively, each element in $\tbbeta$ represents some measure of how well the original data at the input layer explains the variation between observations in $\by$. Moreover, under this approach, if two predictors $\x_r$ and $\x_s$ are almost perfectly collinear, then the corresponding effect sizes will also be very similar since $\text{cov}(\x_r,\f)\approx\text{cov}(\x_s,\f)$. To build a better intuition for identifiability under this covariance projection, recall simple linear regression where ordinary least squares (OLS) estimates are unique modulo the span of the data \citep{Wold:1984aa}. A slightly different issue will arise for the effect size analogues computed via Equation \eq{ESA}, where now two estimates are unique modulo the span of a vector of ones, or $span\{\bm{1}\}$. We now make the following formal statement.

\begin{claim}
Two effect size analogues computed via the covariance projection operators, $\tbbeta_1 = \text{cov}(\bX,\bm{f}_1)$ and $\tbbeta_2 = \text{cov}(\X,\bm{f}_2)$, are equivalent if and only if the corresponding functions are related by $\f_1 = \f_2+c\bm{1}$, where $\bm{1}$ is a vector of ones and $c$ is some arbitrary constant.
\end{claim}
\noindent The proof of this claim is trivial and follows directly from the covariance being invariant with respect to changes in location. Other proofs connecting this effect size to classic statistical measures can be found in the Supplementary Material.

\subsection{Closed-Form Relative Centrality Measures for Bayesian Neural Networks} \label{subsec:cf-for-cent-measures}

Under our modeling assumptions, the posterior distribution of $\tbbeta$ is multivariate normal with an empirical mean vector $\bmu$ and positive semi-definite covariance/precision matrix $\bOmega = \bLambda^{-1}$. Given these values, we may partition conformably for the $j$-th input variable such that
\begin{align*}
\bmu = \beginmat{cc} \mu_j \\ \bmu_{-j} \endmat, \quad \quad \bm{\Omega} = \beginmat{cc} \omega_{j} & \bm{\omega}^{\T}_{-j}\\ \bm{\omega}_{-j} & \bm{\Omega}_{-j}\endmat, \quad \quad \bm{\Lambda} = \beginmat{cc} \lambda_{j} & \bm{\lambda}^{\T}_{-j}\\ \bm{\lambda}_{-j} & \bm{\Lambda}_{-j}\endmat.
\end{align*}
With these normality assumptions, after conditioning on $\tbeta_j = 0$, Equation \eqref{eq:KLD} for the RATE criterion has the following closed-form solution
\begin{equation} \label{eq:kld-j}
\text{KLD}_j = \frac{1}{2}\bigg[\tr(\bOmega_{-j}\bLambda_{-j})-\log\left|\bOmega_{-j}\bLambda_{-j}\right|-(p-1)+\delta_j\mu_j^2\bigg] \,,
\end{equation}
where $\tr(\cdot)$ is the matrix trace function, and $\delta_j = \bm{\lambda}^{\T}_{-j}\bm{\Lambda}_{-j}^{-1}\bm{\lambda}_{-j}$ and characterizes the implied linear rate of change of information when the effect of any predictor is absent --- thus, providing a natural (non-negative) numerical summary of the role of each $\widetilde{\beta}_j$ plays in defining the full joint posterior distribution. In a dataset with a reasonably large number of features, the term $\tr(\bOmega_{-j}\bLambda_{-j})-\log\left|\bOmega_{-j}\bLambda_{-j}\right|-(p-1)$ remains relatively equal for each input variable and, thus, makes a negligible contribution to when determining the variable importance \citep{crawford2018variable}. Therefore, in practice, we compute RATE measures using
\begin{equation} \label{eq:kld-j-approx}
    \text{KLD}_j \approx \delta_j\mu_j^2/2 \,.
\end{equation}
Note that the scalability of the RATE calculation in Equation \eq{eq:kld-j-approx} (which includes a feature's posterior mean and the joint covariance matrix) is $\mathcal{O}(pn^2 + p^2n + p^4)$ for $n$ observations and $p$ variables. Hence, the leading order term is $\mathcal{O}(p^4)$ which is driven by $p$ independent $\mathcal{O}(p^3)$ operations of solving the $p-1$ dimensional linear systems $\delta_j$ for $j=1,\ldots,p$. This restricts the current implementation of RATE to datasets of size $n \lessapprox 10^5$ and $p \lessapprox 10^4$ if the system is solved. Fortunately, the matrix $\bm{\Lambda}_{-j}^{-1}$ differs by only a single row and column between consecutive values of $j$, meaning that low-rank updates can be used to solve $\delta_j = \bm{\lambda}^{\T}_{-j}\bm{\Lambda}_{-j}^{-1}\bm{\lambda}_{-j}$ in $\mathcal{O}(p^2)$ time using the Sherman-Morrison formula (for a single $j$-th index) \citep{hager1989updating}. Therefore, the overall run complexity of the RATE algorithm can be further reduced to $\mathcal{O}(p^3)$, but we leave this potentially more scalable implementation for future work. Our results in this study indicate that the current implementation of RATE is already faster than comparable mimic models once time for cross-validation is also considered (see Section \ref{sec:results}).

\subsection{Relationship between Relative Centrality and Mutual Information}

To build further intuition about centrality measures, we establish a formal connection between the RATE measure and mutual information (MI). Notice, that by simplifying the definition of mutual information, we have the following
\begin{equation}
\begin{aligned}
\text{MI}(\widetilde{\bm{\beta}}_{-j},\widetilde{\bm{\beta}}_{j})&= \iint p(\widetilde{\bm{\beta}}_{-j},\widetilde{\bm{\beta}}_{j})\,\log\left(\frac{p(\widetilde{\bm{\beta}}_{-j},\widetilde{\bm{\beta}}_{j})}{p(\widetilde{\bm{\beta}}_{-j})p(\widetilde{\bm{\beta}}_{j})}\right)\diff\widetilde{\bm{\beta}}_{-j}\diff\widetilde{\bm{\beta}}_{j}\\
    &=\iint p(\widetilde{\bm{\beta}}_{j})p(\widetilde{\bm{\beta}}_{-j}\cond\widetilde{\bm{\beta}}_{j})\,\log\left(\frac{p(\widetilde{\bm{\beta}}_{-j}\cond \widetilde{\bm{\beta}}_{j})}{p(\widetilde{\bm{\beta}}_{-j})}\right)\diff\widetilde{\bm{\beta}}_{-j}\diff\widetilde{\bm{\beta}}_{j}\\
    &=\int p(\widetilde{\bm{\beta}}_{j})\,\text{KL}\bigg(p(\widetilde{\bm{\beta}}_{-j}\cond \widetilde{\beta}_j)\,\|\,p(\widetilde{\bm{\beta}}_{-j}) \bigg)\diff\widetilde{\bm{\beta}}_{j}\;.
\end{aligned}
\end{equation}
While the RATE criterion compares the marginal distribution $p(\widetilde{\bm{\beta}}_{-j})$ to the conditional distribution $p(\widetilde{\bm{\beta}}_{-j}\cond \widetilde{\beta}_j=0)$ with the effect of the $j$-th predictor being set to zero, the mutual information criterion compares $p(\widetilde{\bm{\beta}}_{-j})$ to the conditional distribution $p(\widetilde{\bm{\beta}}_{-j}\cond \widetilde{\beta}_j)$ for all the possible values of $\widetilde{\beta}_j$. Whenever the effect size analogue follows a normal distribution $\widetilde{\bm{\beta}}\sim\mathcal{N}(\bm{\mu},\bm{\Omega})$, the unnormalised RATE criterion for the $j$-th variable is given by Equation \eqref{eq:kld-j}. In the same setting, the mutual information can also be computed analytically as
\begin{align} \label{eq:mutinf}
\text{MI}(\widetilde{\bm{\beta}}_{-j},\widetilde{\bm{\beta}}_{j})=\frac{1}{2}\;\log\left(\delta_j|\bm{\Omega}_{-j}||\bm{\Omega}|^{-1}\right) \,,
\end{align} 
where the mutual information criterion is equal to $0$ if and only if $\widetilde{\bm{\beta}}_{-j}$ and $\widetilde{\bm{\beta}}_{j}$ are independent. To see the difference between the two information theoretic measures in Equations \eqref{eq:kld-j} and \eqref{eq:mutinf}, notice that $\text{MI}(\widetilde{\bm{\beta}}_{-j},\widetilde{\bm{\beta}}_{j})$ only depends on the values of the covariance/precision matrix $\bm{\Omega} = \bLambda^{-1}$. This is in contrast to the RATE criterion which also takes the posterior mean (or marginal effect) of input features $\bm{\mu}$ into account when determining variable importance. Therefore, if a feature is only marginally associated with an outcome but does not have any significant covarying relationships with other variables in the data, RATE will still identify this feature as being an important predictor.


\section{Results} \label{sec:results}

In this section, we first illustrate the performance of our interpretable Bayesian neural network framework via simulation studies in both the regression and binary classification settings. Here, the goal is to show how determining variable importance for a trained neural network with the RATE measure compares with commonly used mimic modeling techniques in the field. Finally, we examine the potential of our approach in two real datasets from computer vision and natural language processing, respectively.

\subsection{Simulation Studies}\label{subsec:sim-studies}

For all assessments with synthetic data, we consider a simulation design that is often used to explore the power statistical methods \citep{crawford2018bayesian,crawford2018variable}. Let $\bX$ denote a design matrix of $n$ independent observations with $p = 100$ predictor variables. We consider that only a subset of $30$ randomly chosen variables are causal. This subset of causal features is divided into three groups of equal size: features in the two first groups ($\mathcal{C}_1$ and $\mathcal{C}_2$) are involved in pairwise interactions, while features in the third group ($\mathcal{C}_3$) only have additive effects. More precisely, we consider the following generalized linear model for every observation:
\begin{align*}
\by = \sigma(\f), \quad \quad \f = \sum_{j\in\mathcal{C}}\bx_j\beta_j + \sum_{j\in\mathcal{C}_1}\sum_{k\in\mathcal{C}_2}(\bx_{j}\odot\bx_{k})\alpha_{jk} + \bvarepsilon, \quad \quad \bvarepsilon\sim \N(\bm{0},\tau^2\bI) \,,
\end{align*}
where $\sigma(\cdot)$ is a link function, $\bx_{j}\odot\bx_{k}$ denotes an element-wise multiplication between the $j$-th and $k$-th feature vectors, and $\mathcal{C}=\mathcal{C}_1\cup\mathcal{C}_2\cup\mathcal{C}_3$ denotes the union of indices for all causal features. Note that  we only allow interactions to occur between groups, so that features in the first group interact with features in the second group, but do not interact with variables within their own group. The additive and interaction effects are independently drawn from normal distributions: $\bbeta \sim \N(\bm{0},\sigma^2_\beta\bI)$ and $\balpha\sim\N(\bm{0},\sigma^2_\alpha\bI)$, respectively. The variance components $\sigma^2_\beta$ and $\sigma^2_\alpha$ are scaled such that the additive and interaction effects explain 30\% of the total variance in the response variable, while the remaining 40\% is explained by the error term. In other words, 
\begin{align*}
\V\left[\sum_{j\in\mathcal{C}}\bx_j\beta_j\right]= \V\left[\sum_{j\in\mathcal{C}_1}\sum_{k\in\mathcal{C}_2}(\bx_{j}\odot\bx_{k})\alpha_{jk}\right]= 0.3\quad \text{ and }\quad \V[\bvarepsilon]=\tau^2=0.4 .
\end{align*}
We consider a few simulation scenarios while varying the sample size $n\in\{3\times10^{3},\;1\times10^{4},\;3\times10^{4},\;1\times10^{5}\}$. We also consider two link functions: \textit{(i)} the identify function, so that $\by = \f$, for simulating continuous response variables; and \textit{(ii)} a combination of the sigmoid function and thresholding, $\by = \mathbbm{1}[[1+\exp\{-\f\}]^{-1} \ge 0.5]$, for generating binary classification labels. Additional noise was also added by flipping 10\% of the labels in the binary classification setting.

For the simulations considered here, we use 70\% of the observations samples to train a Bayesian neural network with two (deterministic) hidden layers of sizes 32 and 16 units respectively and rectified linear unit (ReLU) activations. We assume a mean-field diagonal Gaussian approximating distribution over the final layer weights and train the network using variational inference. We then evaluate variable importance for the Bayesian neural network using RATE and two other \textit{post-hoc} measures:
\begin{enumerate}
    \item The Mean Decrease Impurity variable importance of a random forest (RF) mimic model \citep{breiman2001random} trained to mimic the mean soft predictions on the training data;
    \item The Mean Decrease Impurity variable importance of a gradient boosting machine (GBM) mimic model \citep{friedman2001greedy} trained to mimic the mean soft predictions on the training data.
\end{enumerate}
Note that, for simulations where the task is binary classification, RATE values are computed using the posterior of the pre-sigmoid function values $\f$ and the RF mimic model is trained to predict the mean latent class probabilities $\Pr[y_i = 1]$, for $1\leq i\leq n$ over Monte Carlo samples from the predictive posterior. More details on the Bayesian neural network training procedure and mimic model cross-validation can be found in the Supplementary Material. 

We evaluate each method's ability to effectively prioritize the causal features in 25 different simulated datasets for each simulation set-up. The criteria we use compares receiver operating characteristic (ROC) curves where the false positive rate (FPR) is plotted against the rate at which true variables are identified by each method (TPR) (Figures \ref{Fig2} and \ref{Fig3}). This is further quantified by comparing the area under curves (AUC) across experiments where a higher value denotes better accuracy in prioritizing important input features (Tables \ref{Tab1} and \ref{Tab2}). For the smallest sample size ($n=3\times10^3$), RATE and the mimic models have similar performances; although in the classification scenario (Figure~\ref{Fig3}a), this is close to random guessing as there was insufficient data for the Bayesian neural network to exhibit good predictive performance. For the moderately larger sample sizes (i.e., $n \ge 1\times10^4$), both RATE and the GBM mimic begin to separate as the best methods with consistently higher median AUCs than the RF mimic model. The random forest also exhibits far more variability in its performance. While the GBM model is able to match RATE, its computational time is two orders of magnitude larger and also requires hand-tuning for the cross-validation procedure (see Figure~\ref{Fig4}). A notable and practical advantage of RATE is that it does not require tuning any hyper-parameters. 

\begin{figure}[hbt!]
\centering
\subfigure[Sample Size: $n = 3\times 10^{3}$]{
\includegraphics[width = 0.45\textwidth]{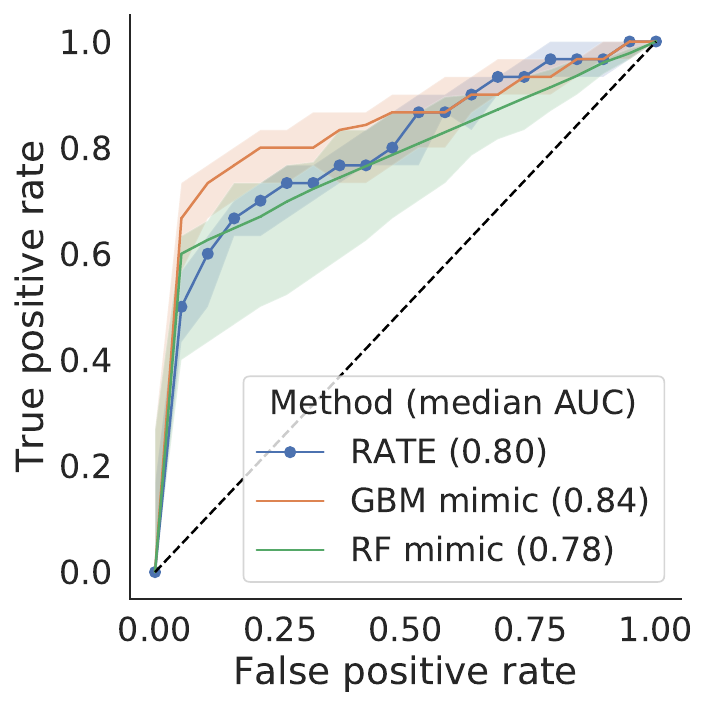}
}
\subfigure[Sample Size: $n = 1\times 10^{4}$]{
\includegraphics[width = 0.45\textwidth]{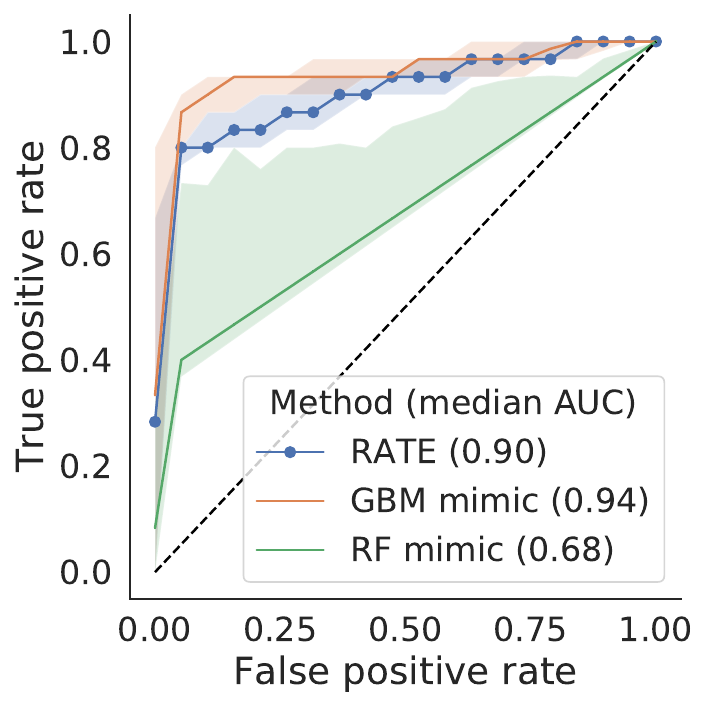}
}
\subfigure[Sample Size: $n = 3\times 10^{4}$]{
\includegraphics[width = 0.45\textwidth]{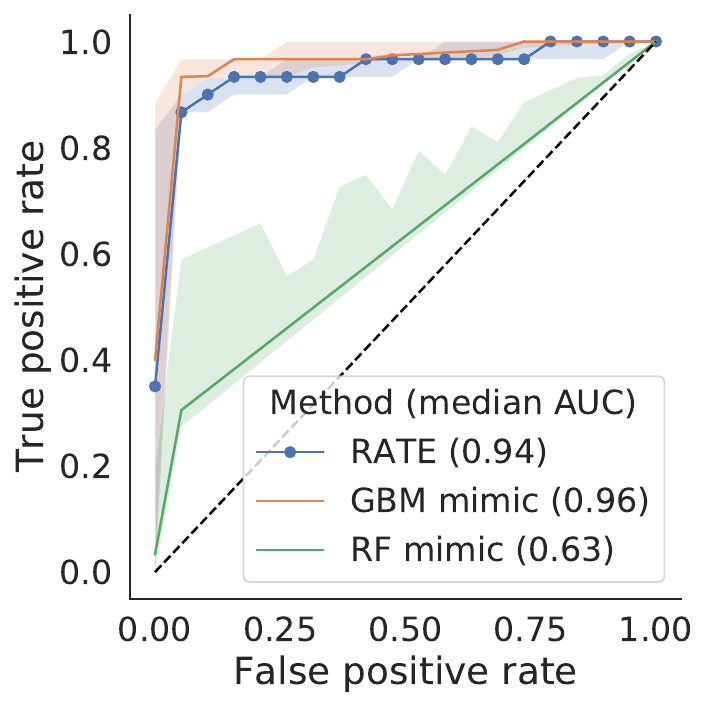}
}
\subfigure[Sample Size: $n = 1\times 10^{5}$]{
\includegraphics[width = 0.45\textwidth]{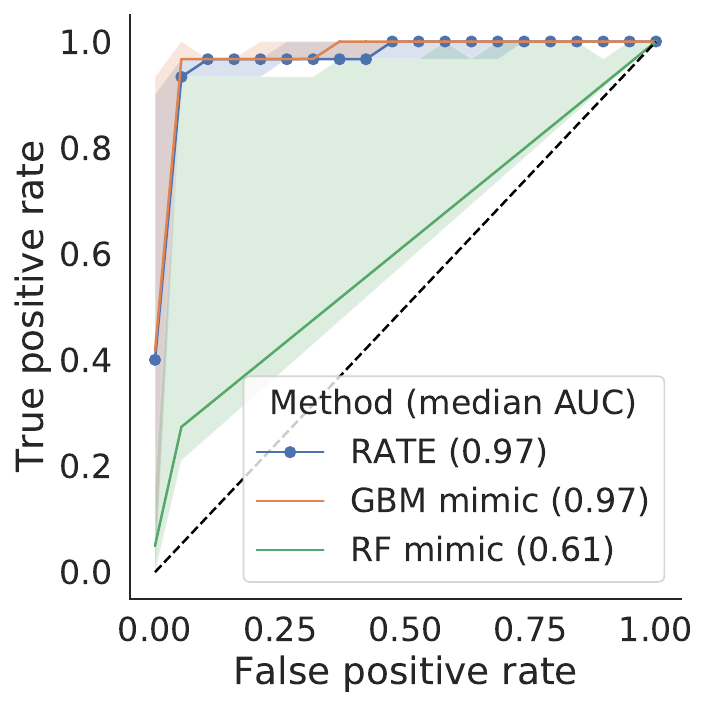}
}
\caption{Receiver operating characteristic (ROC) curves assessing the effectiveness of RATE (blue) as a variable importance measure in simulations with continuous responses (i.e., regression setting). Here, we compare the performance of random forest (RF; green) and gradient boosting machine (GBM; orange) mimic models as a baseline. Depicted are the median curves taken over 25 repeated experiments, where causal features and their corresponding effect sizes are resampled for each repeat. Area under the curve (AUC) is reported to facilitate comparisons. The shaded areas indicate the 95\% confidence intervals for each curve using 1,000 bootstrapped curves.}
\label{Fig2}
\end{figure}

\begin{figure}[hbt!]
\centering
\subfigure[Sample Size: $n = 3\times 10^{3}$]{
\includegraphics[width = 0.45\textwidth]{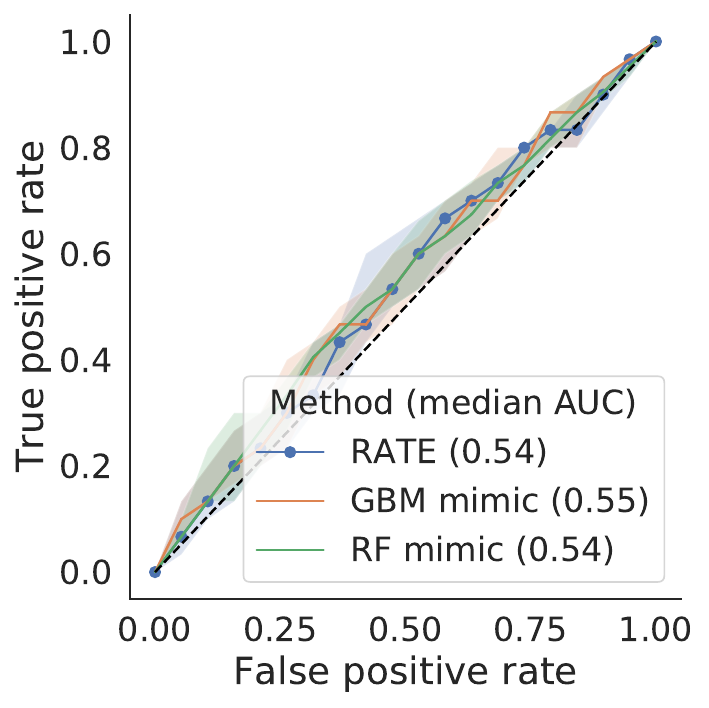}
}
\subfigure[Sample Size: $n = 1\times 10^{4}$]{
\includegraphics[width = 0.45\textwidth]{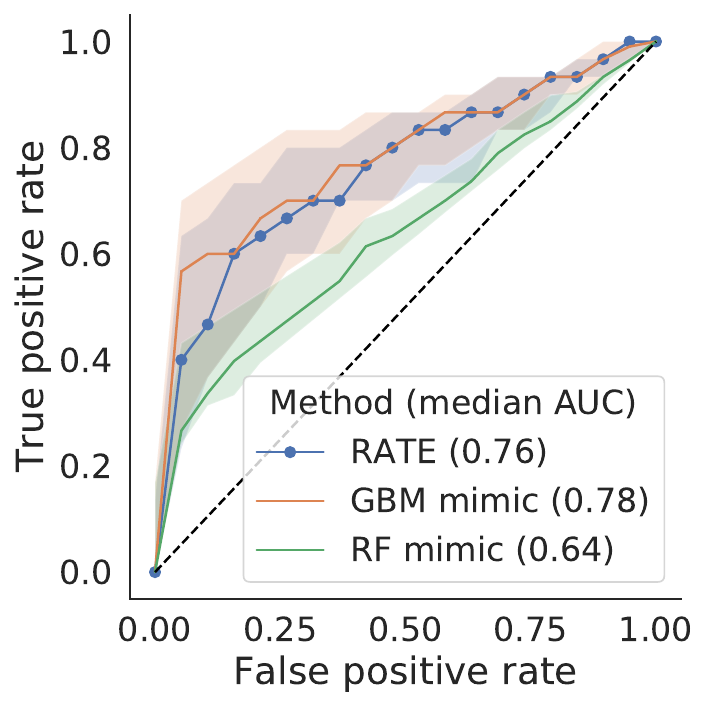}
}
\subfigure[Sample Size: $n = 3\times 10^{4}$]{
\includegraphics[width = 0.45\textwidth]{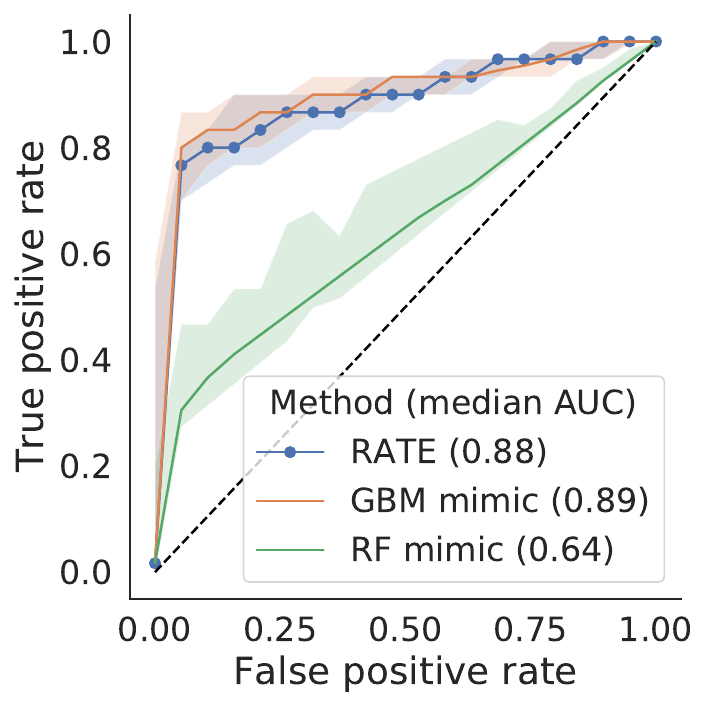}
}
\subfigure[Sample Size: $n = 1\times 10^{5}$]{
\includegraphics[width = 0.45\textwidth]{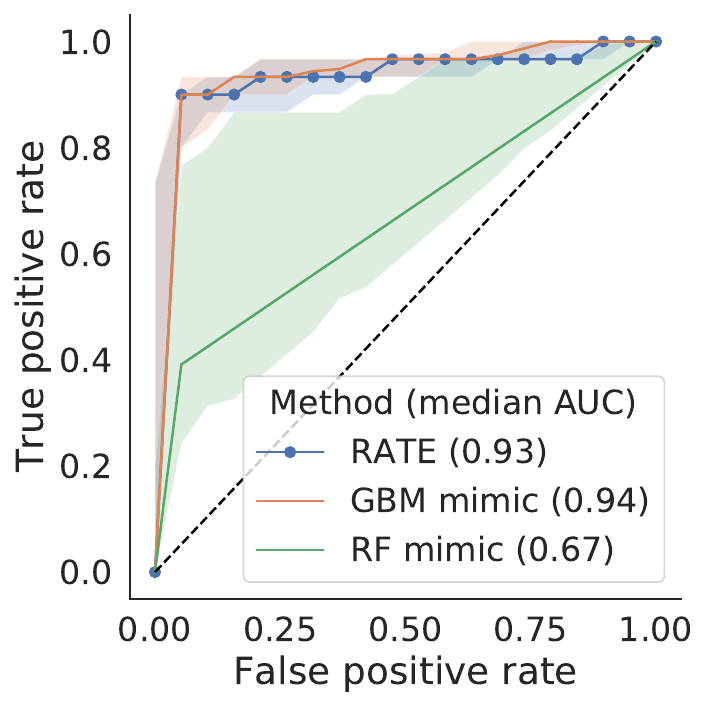}
}
\caption{Receiver operating characteristic (ROC) curves assessing the effectiveness of RATE (blue) as a variable importance measure in simulations with binary labels (i.e., classification setting). Here, we compare the performance of random forest (RF; green) and gradient boosting machine (GBM; orange) mimic models as a baseline. Depicted are the median curves taken over 25 repeated experiments, where causal features and their corresponding effect sizes are resampled for each repeat. Area under the curve (AUC) is reported to facilitate comparisons. The shaded areas indicate the 95\% confidence intervals for each curve using 1,000 bootstrapped curves.}
\label{Fig3}
\end{figure}

\begin{table}[ht]
\caption{Median area under the curves (AUC) assessing the effectiveness of RATE as a variable importance measure in simulations with continuous responses (i.e., regression setting). Here, we compare the performance of random forest (RF) and gradient boosting machine (GBM) mimic models as a baseline. Values in the parentheses are the 95\% confidence intervals for each curve taken over 25 repeated experiments. These AUC values correspond to the curves in Figure \ref{Fig2}.}
\begin{center}
\begin{small}
\begin{sc}
\begin{tabular}{lcccc}
\toprule
Method & $n=3,000$ & $n=10,000$ & $n=30,000$ & $n=100,000$ \\
 \midrule
RATE & 0.80 (0.75-0.83) & 0.90 (0.87-0.93) & 0.94 (0.92-0.97) & 0.97 (0.94-0.99) \\
GBM mimic & 0.84 (0.78-0.87) & 0.94 (0.90-0.96) & 0.96 (0.94-0.99) & 0.97 (0.96-0.99) \\
RF mimic & 0.78 (0.69-0.82) & 0.68 (0.66-0.84) & 0.63 (0.61-0.74) & 0.61 (0.58-0.93) \\
\bottomrule
\end{tabular}
\end{sc}
\end{small}
\end{center}
\label{Tab1}
\end{table}

\begin{table}[hbt!]
\caption{Median area under the curves (AUC) assessing the effectiveness of RATE as a variable importance measure in simulations with binary labels (i.e., classification setting). Here, we compare the performance of random forest (RF) and gradient boosting machine (GBM) mimic models as a baseline. Values in the parentheses are the 95\% confidence intervals for each curve taken over 25 repeated experiments. These AUC values correspond to the curves in Figure \ref{Fig3}.}
\begin{center}
\begin{small}
\begin{sc}
\begin{tabular}{lcccc}
\toprule
Method & $n=3,000$ & $n=10,000$ & $n=30,000$ & $n=100,000$ \\
 \midrule
RATE & 0.53 (0.50-0.58) & 0.77 (0.68-0.85) & 0.89 (0.84-0.93) & 0.93 (0.91-0.96) \\
GBM mimic & 0.53 (0.50-0.58) & 0.79 (0.68-0.87) & 0.90 (0.87-0.94) & 0.94 (0.93-0.98) \\
RF mimic & 0.53 (0.49-0.57) & 0.71 (0.66-0.80) & 0.69 (0.65-0.75) & 0.75 (0.70-0.87) \\
\bottomrule
\end{tabular}
\end{sc}
\end{small}
\end{center}
\label{Tab2}
\end{table}

\begin{figure}[hbt!]
\centering
\subfigure[Regression]{
\includegraphics[width = 0.47\textwidth]{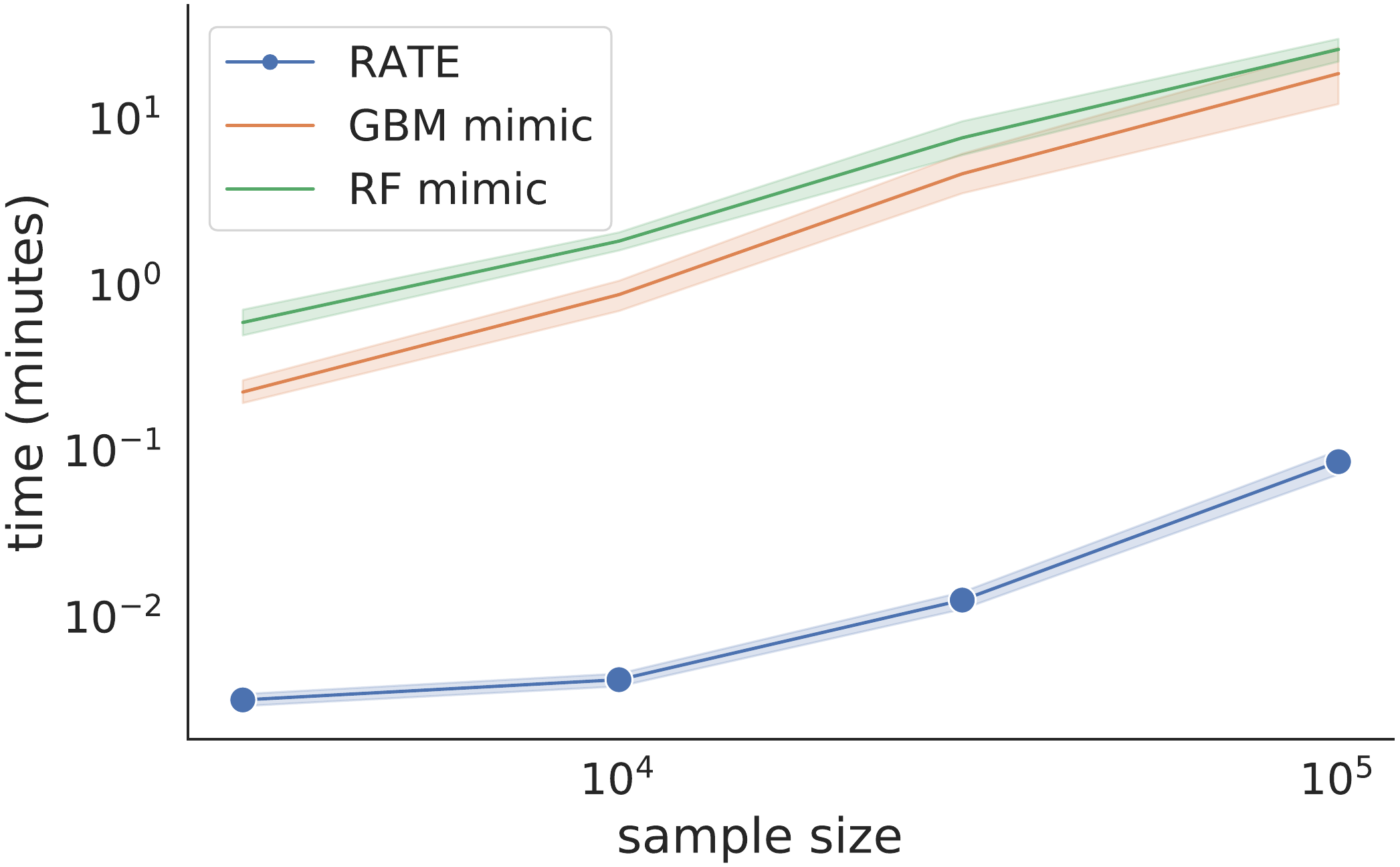}
}
\subfigure[Classification]{
\includegraphics[width = 0.47\textwidth]{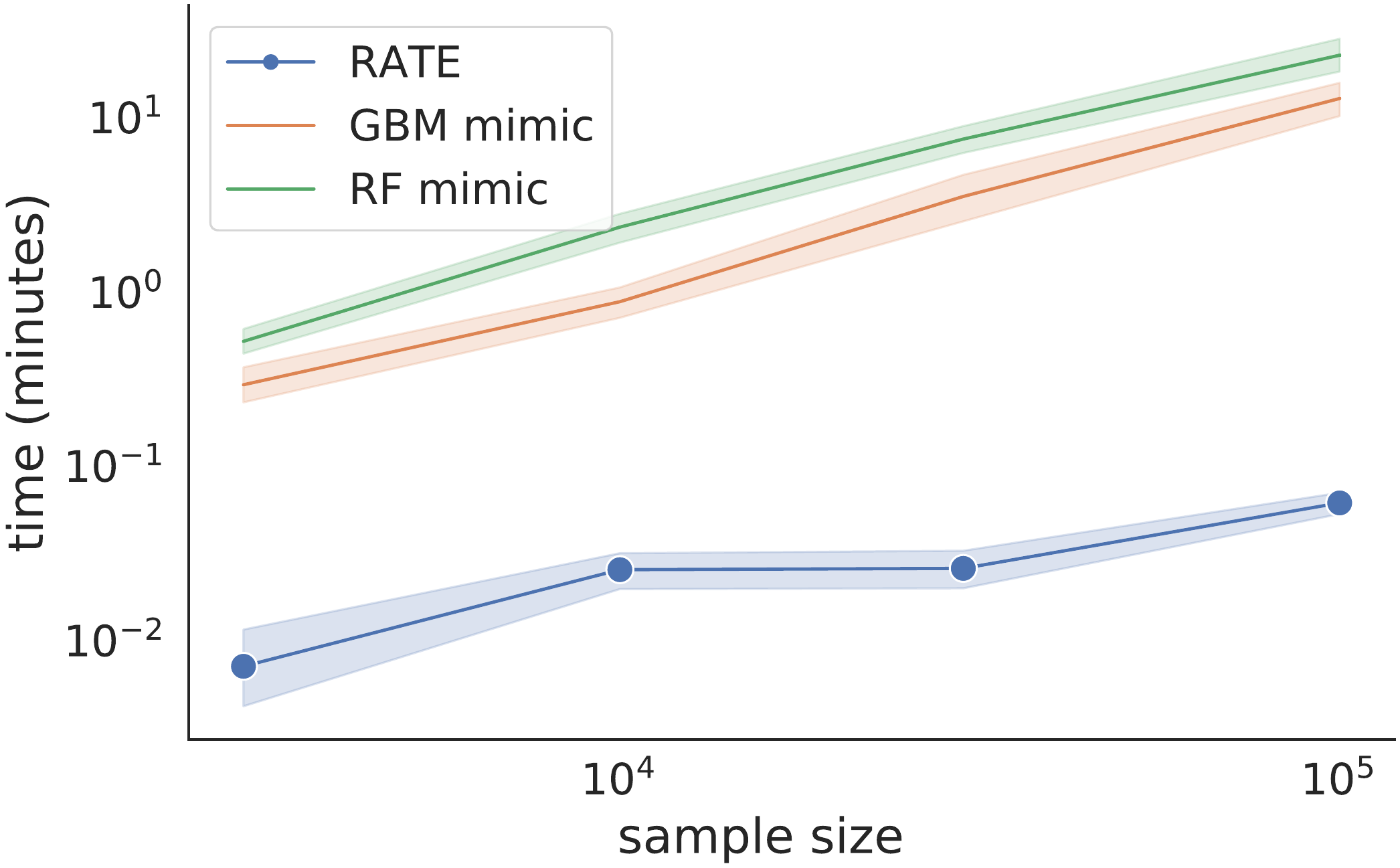}
}
\caption{The run times (in minutes) for RATE (blue) and the competing mimic models in the simulation studies. Results are shown in both the (a) regression and (b) classification settings. All computations were performed using 32 cores. The random forest (RF; green) and gradient boosting machine (GBM; orange) mimic models were selected using 5-fold cross-validation and a random hyper-parameter search over 30 different models. Shaded regions depict standard errors across the simulated replicates.}
\label{Fig4}
\end{figure}


\subsection{Binary Image Classification using MNIST} \label{subsec:results-mnist}

Here, we demonstrate the utility of RATE in two binary classification tasks from the MNIST dataset \citep{lecun1998mnist}. The Bayesian neural network used in this analysis contains a single convolution layer, followed by two fully-connected layers (see Supplementary Material for further details). We would like to re-emphasize that the focus here is on \textit{post-hoc} interpretations of a trained network, so the network architecture and training procedures were not optimized for predictive performance. As in the previous section, we compute variable (pixel) importance using \textit{(i)} RATE values, \textit{(ii)} a random forest (RF) mimic model, and \textit{(iii)} a gradient boosting machine (GBM) mimic model. In addition, we also included some local interpretability methods (collectively known as saliency maps) that attribute pixel importance using the gradient of the network output with respect to each pixel. The drawbacks of saliency-based methods have been well-documented \citep{adebayo2018sanity,kindermans2019reliability,ghorbani2019interpretation}, but they are included here due to their popularity in computer vision. See \cite{ancona2017towards} for an analysis and comparison of these saliency methods.

A direct comparison to RATE is not possible as saliency maps are local methods and are therefore used to explain a network's prediction on a single image. However, we can assign global importance to a pixel by taking the mean absolute value of its local importance over a set of observations. For example, the simplest saliency map attributes the partial derivative $\partial y_i/\partial x_{ij}$ as the importance of the $j$-th pixel in the $i$-th image. We then assign global importance using $\sum_{i=1}^n \left| \partial \y/\partial \x_j \right|/n$. In addition to this ``vanilla'' gradient, we also compare our approach to the \textit{(i)} Integrated Gradient \citep{sundararajan2017axiomatic}, \textit{(ii)} Gradient$\odot$Input (where again $\odot$ denotes element-wise product) \citep{shrikumar2016not}, and \textit{(iii)} $\varepsilon$-Layer-wise Relevance Propagation ($\varepsilon$-LRP) \citep{bach2015pixel}.

In both sets of analyses, we show that RATE is able to at least match the best performing mimic models, which are more established approaches to global interpretability (compared to aggregating saliency maps). To the best of our knowledge, there have been no previous studies on the use of aggregated saliency methods for global interpretations of a neural network.

\paragraph{Zeros vs.~Ones.} For the first task, we take just the zeros and ones from MNIST, resulting in a dataset of $n =$ 12,665 training images, each with $p =$ 324 pixels (after cropping). Figure~\ref{Fig5} shows how the pixels are ranked by each method (diagonal plots in Figure~\ref{Fig5}), while also comparing the rankings of the pixels according to each method (lower diagonal plots in Figure~\ref{Fig5}). The pixels identified by RATE are consistent with human intuition: when distinguishing between zero and one, the most important pixels are in the center (where the vertical line of a one would appear) and in a ring (corresponding to the shape of zero). The RF mimic model produces a similar visualization and ranking (Spearman's $\rho=0.79$), with especially strong agreement between the highly-ranked pixels. The results from mean absolute integrated gradient are also similar to RATE (Spearman's $\rho=0.63$), but are less defined. The GBM mimic highlights very few pixels, while the mean absolute gradient produces poorly-defined visualization.

While our results suggest a plausible set of important pixels under visual inspection, a natural followup analysis is to quantitatively assess these findings. As this is a real dataset, we do not have access to the ground truth. However, we can investigate how important each pixel is to a given network when it makes an out-of-sample prediction. To do so, we calculate prediction accuracy as certain pixels in the test images are shuffled, thus de-correlating those pixels from the labels. Figure~\ref{Fig6} shows the test set accuracy as progressively larger subsets of pixels are shuffled, where the pixels are shuffled in the order of their ranking according to each method. A ``good'' variable importance method ranks the pixels in such a way that the test accuracy decreases quickly. For this simple problem, the saliency methods (i.e., the mean absolute gradient, integrated gradient, gradient$\odot$input , and $\varepsilon$-LRP) perform best and lead to the steepest decrease in test accuracy. RATE and the RF mimic model lead to similar (but less steep decreases in test accuracy), while the GBM mimic performs the worst of all the methods. All methods we consider are better than shuffling the pixels at random.

\begin{figure}[hbt!]
\centering
\includegraphics[width = 0.8\textwidth]{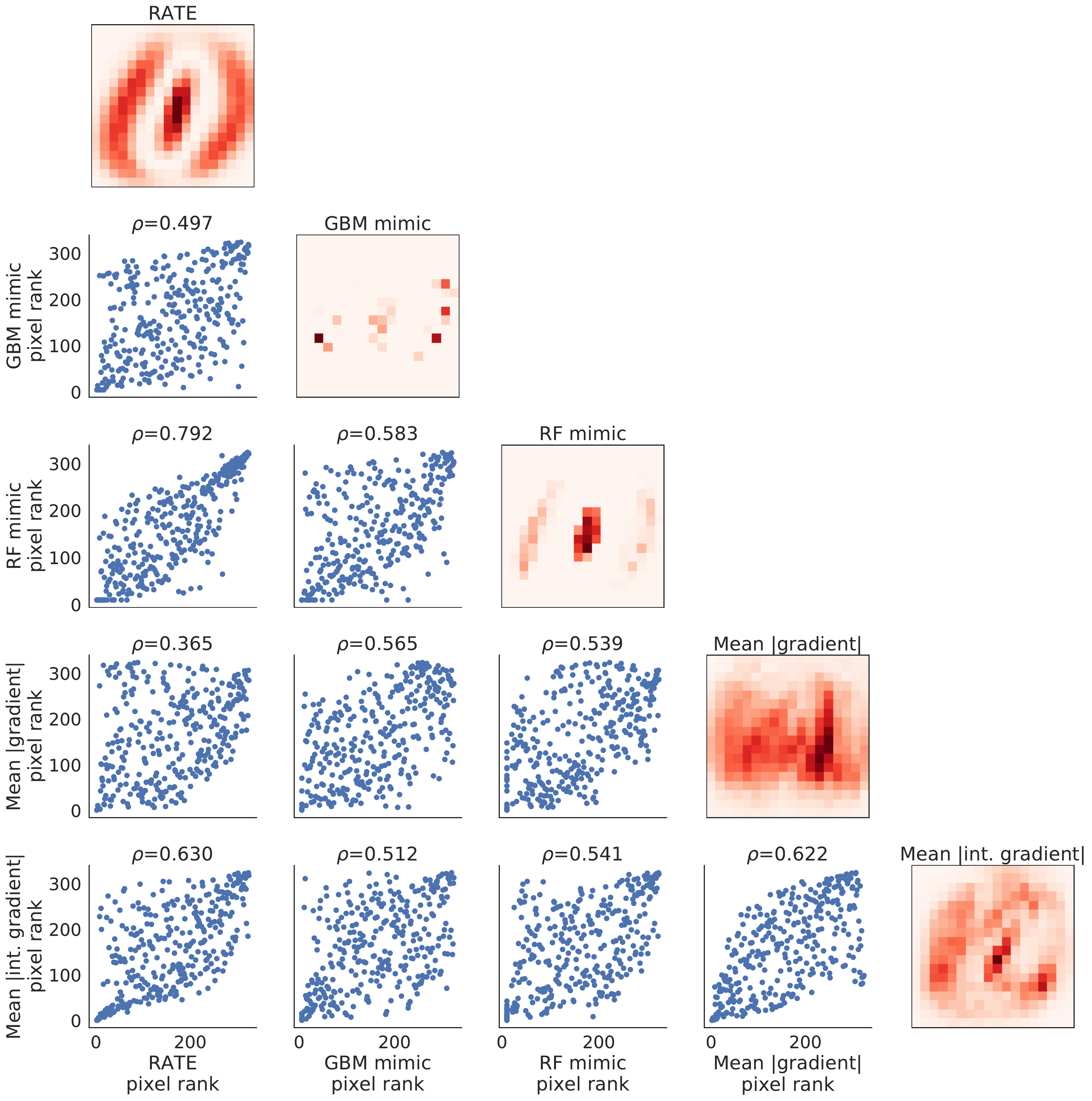}
\caption{Distinguishing zeros and ones in MNIST. Pixel importance (darker pixels are more important) according to \textit{(i)} RATE, \textit{(ii)} a random forest (RF) mimic model, \textit{(iii)} a gradient boosting machine (GBM) mimic model, \textit{(iv)} mean absolute saliency map, and \textit{(v)} mean absolute integrated gradient. Diagonal plots show the importance of each pixel and off-diagonal plots compare the pixel rankings according to the different methods using scatter plots and Spearman's correlation coefficient $\rho$.  Two saliency methods (mean absolute gradient$\odot$input and mean absolute $\varepsilon$-LRP) are omitted as their pixel importances are almost identical to the mean absolute integrated gradient (Spearman's $\rho\ge0.99$). The pixel visualizations for these methods can be found in the Supplementary Material.}
\label{Fig5}
\end{figure}

\begin{figure}[hbt!]
\centering
\includegraphics[width = 0.8\textwidth]{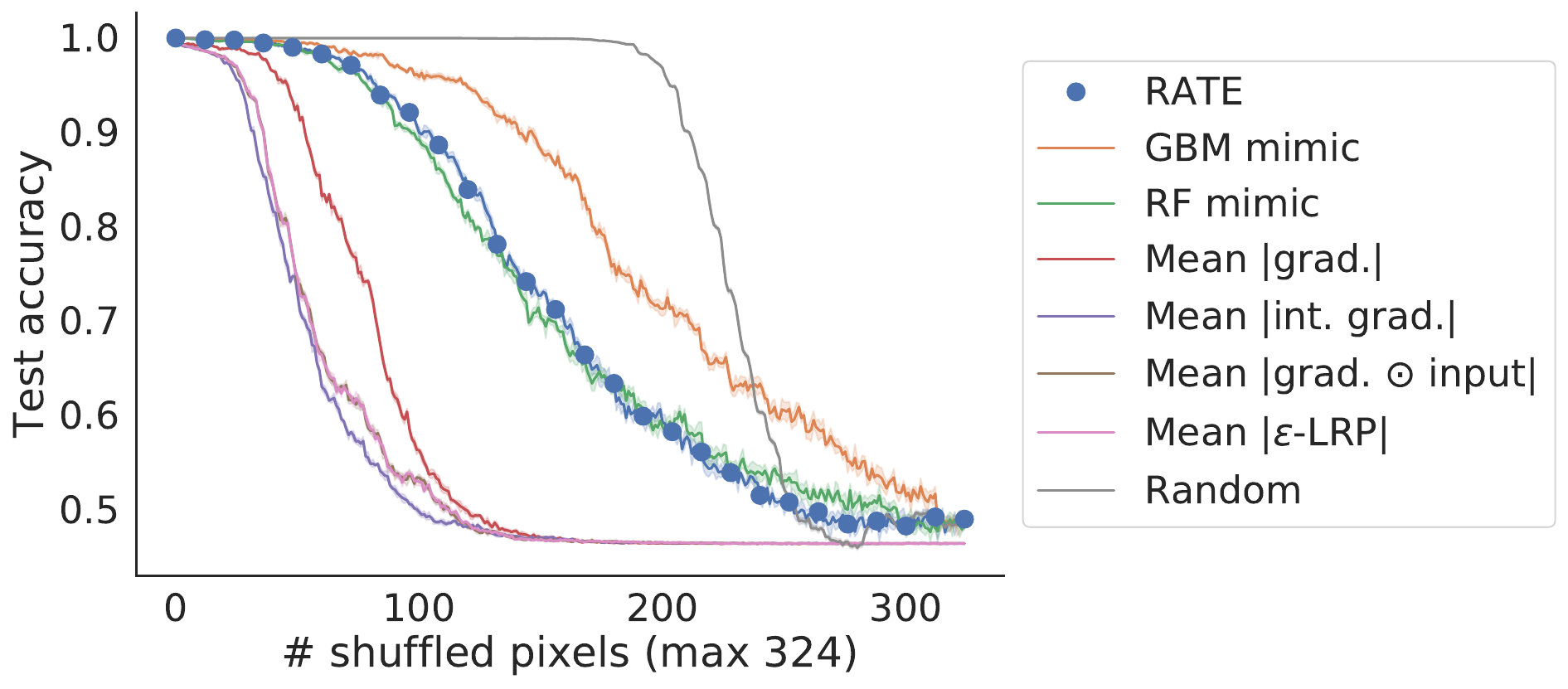}
\caption{Distinguishing zeros and ones in MNIST. The decrease in test accuracy as pixels are shuffled according to their importance (the most highly ranked pixels are shuffled first). A ``good'' variable importance method ranks the pixels in such a way that the test accuracy decreases quickly. Shuffling the pixels at random is also included as a baseline for comparison.}
\label{Fig6}
\end{figure}

\paragraph{Evens vs.~Odds.} We now present additional results for the more difficult binary classification problem of classifying odd and even digits. For this analysis, we used the full MNIST dataset of 60,000 training and 10,000 test images. Once again, each image had $p=$ 324 pixels (after cropping) and we compare RATE to the same six competing methods (see Figure~\ref{Fig7}). As this is a more complex problem the quality of pixel importance is more difficult to evaluate visually. However, RATE produces the most well-defined pixel importance as the mimic models place high importance on a very small number of pixels and the saliency methods placing high importance on a large number of pixels. In the previous problem of analyzing zeros and ones (in which the two classes are linearly separable), the saliency methods were clearly the best performers; however in this setting, RATE, the mean absolute gradient, and the GBM mimic model perform the best (Figure~\ref{Fig8}). The remaining methods (i.e., the RF mimic, mean absolute integrated gradient, gradient$\odot$input and $\varepsilon$-LRP) all perform very similarly. Once again, all the methods decrease the test accuracy more steeply than shuffling pixels at random.

\begin{figure}[hbt!]
\centering
\includegraphics[width = 0.8\textwidth]{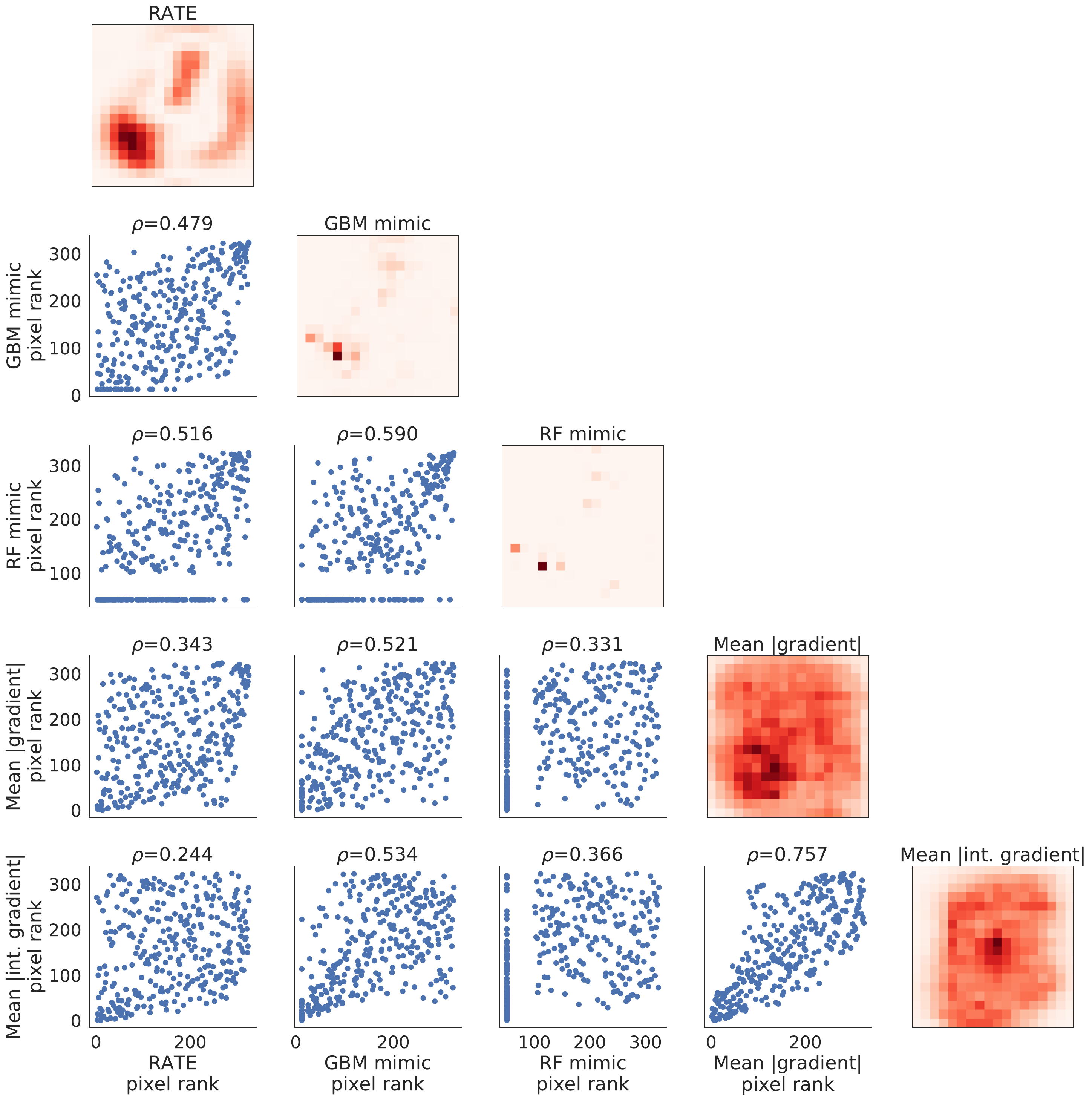}
\caption{Distinguishing evens and odds in MNIST. Pixel importance (darker pixels are more important) according to \textit{(i)} RATE, \textit{(ii)} a random forest (RF) mimic model, \textit{(iii)} a gradient boosting machine (GBM) mimic model, \textit{(iv)} mean absolute saliency map, and \textit{(v)} mean absolute integrated gradient. Diagonal plots show the importance of each pixel and off-diagonal plots compare the pixel rankings according to the different methods using scatter plots and Spearman's correlation coefficient $\rho$.  Two saliency methods (mean absolute gradient$\odot$input and mean absolute $\varepsilon$-LRP) are omitted as their pixel importances are almost identical to the mean absolute integrated gradient (Spearman's $\rho\ge0.99$). The pixel visualizations for these methods can be found in the Supplementary Material.}
\label{Fig7}
\end{figure}

\begin{figure}[hbt!]
\centering
\includegraphics[width = 0.8\textwidth]{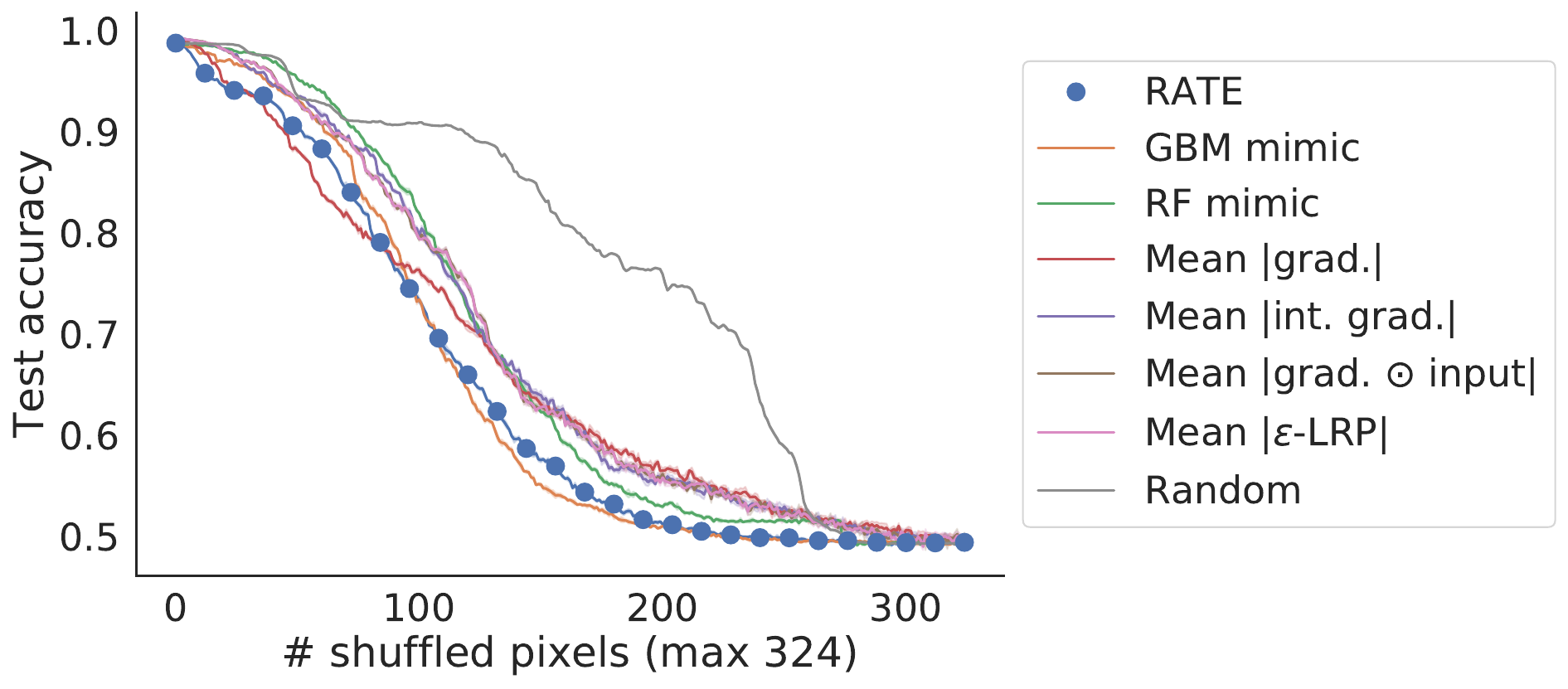}
\caption{Distinguishing evens and odds in MNIST. The decrease in test accuracy as pixels are shuffled according to their importance (the most highly ranked pixels are shuffled first). A ``good'' variable importance method ranks the pixels in such a way that the test accuracy decreases quickly. Shuffling the pixels at random is also included as a baseline for comparison.}
\label{Fig8}
\end{figure}


\subsection{Sentiment Analysis using the Large Movie Review Dataset}

We now present an example of RATE for binary classification in natural language processing. Here, we use the Large Movie Review (IMDB) Dataset, which consists of 50,000 reviews labeled as having positive or negative sentiment \citep{maas2011learning}. The reviews were encoded using term frequency-inverse document frequency (TF-IDF), initially retaining the 1,020 most commonly occurring words in the corpus. For our analysis, we exclude the 20 most common words, resulting in a final dataset with 1,000 words. We split the data into 70\% training and 30\% test sets, and then train a Bayesian neural network with 3 fully-connected, 128-unit hidden layers using the Adam optimizer with a learning rate $1\times 10^{-3}$ \citep{kingma2014adam} (Supplementary Material).

Variable importance for the Bayesian neural network was then calculated using RATE, as well as the random forest and gradient boosting machine mimic models as described in the previous sections. Since we consider a bag-of-words encoding, each input variable in the network corresponds to a single word that has a fixed meaning across all reviews in the data. We make this choice because in alternative encodings (based on word embeddings) the interpretation of variables would correspond to a particular position within a sequence rather than a particular word. RATE is not an appropriate variable importance method for such encodings as it relies on variables having a fixed meaning across examples.

Figures \ref{Fig9}a-\ref{Fig9}c display the ten most important words according to each method. The words identified by our approach are mostly associated with negative sentiment. This reflects an established phenomenon from psychology which poses that negative sentiments tend to outweigh positive ones \citep{baumeister2001bad}. Here, RATE and the GBM mimic method exhibit more agreeable word rankings (Spearman's $\rho =$ 0.63) than either do with the RF mimic (Spearman's $\rho =$ 0.31 and 0.38, respectively). Figure~\ref{Fig9}d shows how the test accuracy decreases as words are shuffled (in order of their importance). This result shows that variables identified by RATE and the GBM mimic models had a much larger impact on being able to predict out-of-sample variation. The RF mimic model struggled overall with this task and exhibited poor predictive performance when evaluated on the BNN predicted probabilities for the held-out data ($R^2\approx0.4$ for the RF mimic versus $R^2\approx0.8$ for the GBM mimic), indicating that it had severely underfit the BNN's predicted probability. This resulted in the RF mimic assigning zero importance to 918 of the 1,000 words, which is why its line in Figure~\ref{Fig9}d (green) not extend as far as the lines for the other two methods. This demonstrates part of the difficulty of using tree ensemble mimic models, which require hand-tuning of the cross-validation procedure to produce good results. For this problem, we observed that a cross-validation procedure that worked well for the original data (where the model is trained on the class labels themselves) did not work when used to train a mimic model (where the model is instead trained on the predicted probabilities of the BNN).

\begin{figure}[htb!]
\centering
\subfigure[RATE]{
\includegraphics[width = 0.31\textwidth]{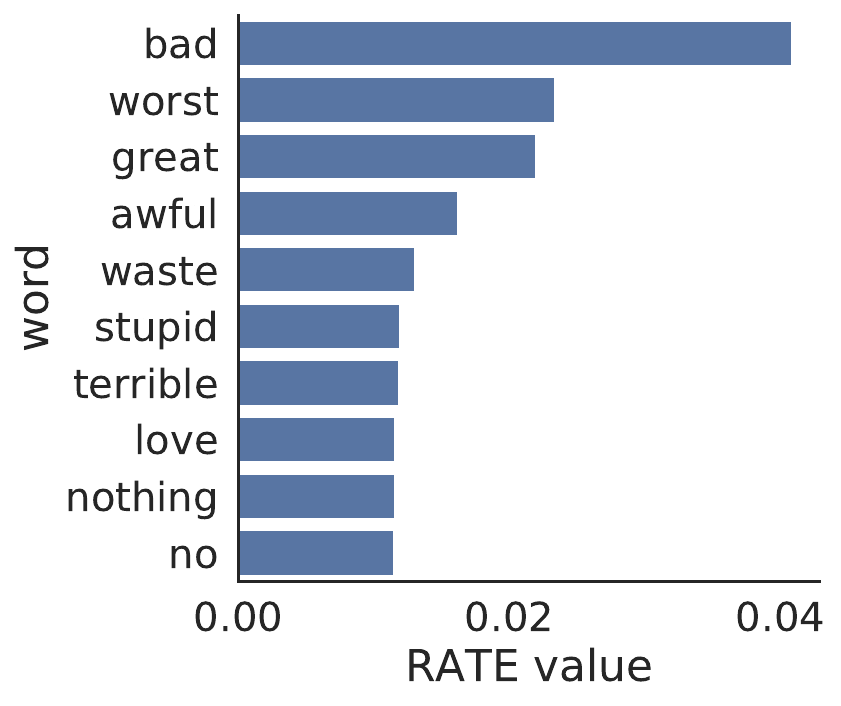}
}
\subfigure[GBM mimic]{
\includegraphics[width = 0.31\textwidth]{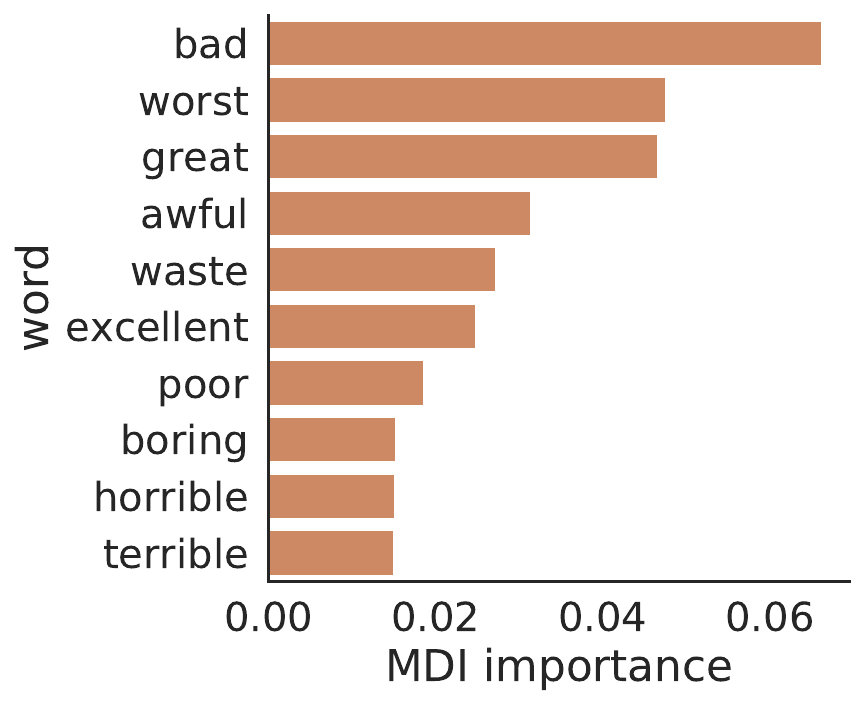}
}
\subfigure[RF mimic]{
\includegraphics[width = 0.31\textwidth]{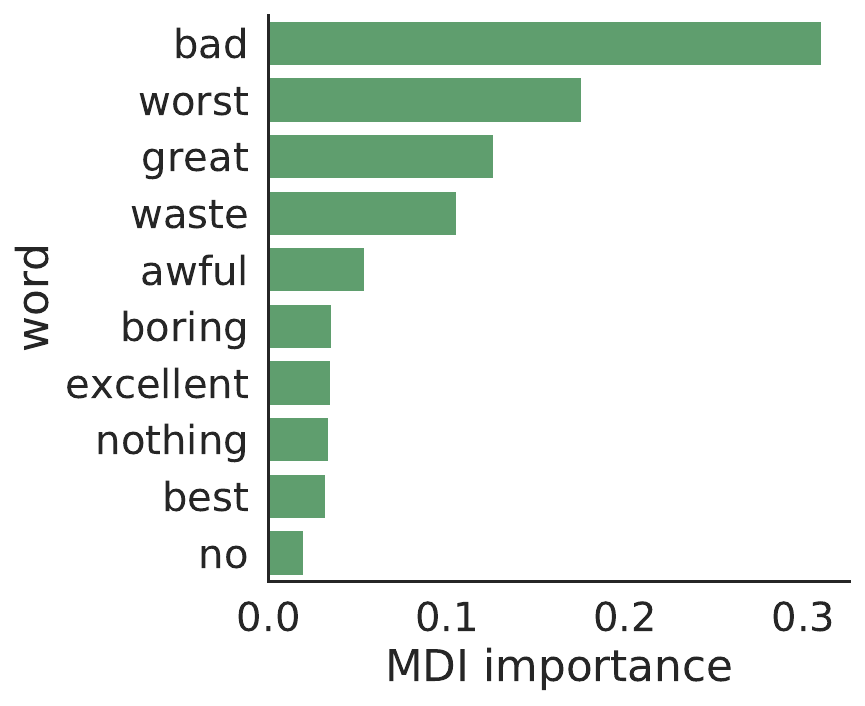}
}
\subfigure[RF mimic]{
\includegraphics[width = 0.85\textwidth]{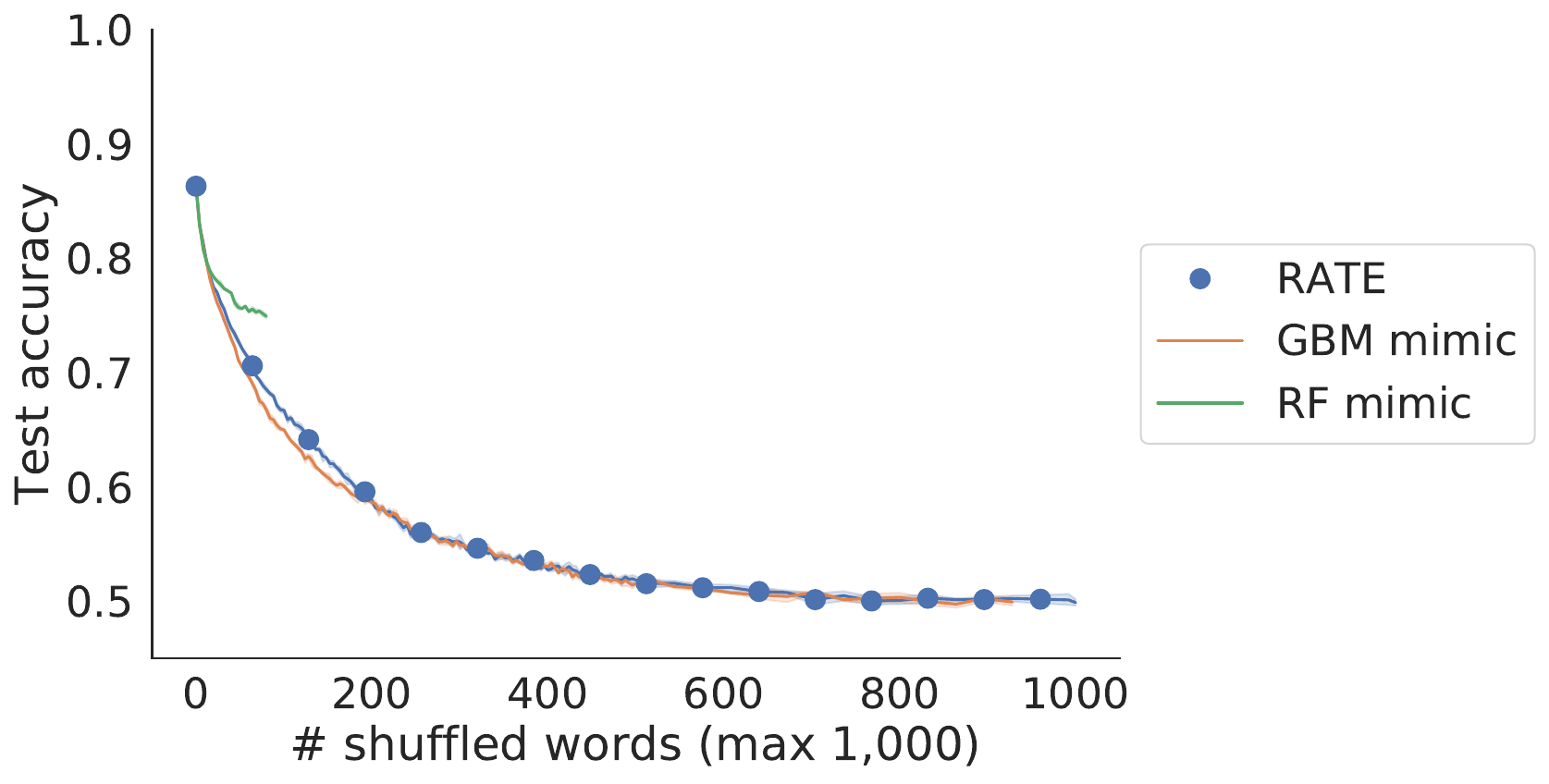}
}
\caption{Variable importance for the Bayesian neural network according to (a) RATE, (b) a random forest (RF) mimic model, and (c) a gradient boosting machine (GBM) mimic model on the Large Movie Review (IMDB) Dataset \citep{maas2011learning}. The latter two approaches use the Mean Decrease Impurity (MDI) importance score to rank words. In panel (d), we show the decrease in test accuracy as words are shuffled according to their importance (the most highly ranked words are shuffled first). A ``good'' variable importance method ranks the pixels in such a way that the test accuracy decreases quickly. Note that the random forest (green) assigns zero importance to 918/1000 words, meaning that its line does not extend beyond 182 shuffled words}
\label{Fig9}
\end{figure}

\section{Relative Centrality Measures for Groups of Variables} \label{sec:grouprate}

Depending on the application setting, one might be interested in assessing the joint global importance for multiple input variables at a time. For example, assume that we have prior knowledge about how sets of variables are related (e.g., a collection of SNPs being inside the boundary of gene) \citep{Wu:2010aa} and we are interested in ranking these groups rather than individual predictors. We may extend the univariate RATE criterion in Equation \eqref{eq:kld-j} for these types of set-based analyses. Let $\mathcal{S}_g$ denote the $g$-th collection of input variables $\{j: j\in\mathcal{S}_g\}$. As done in the univariate case, once we have access to draws from the posterior distribution of the effect size analogue $\tbbeta$, we may conformably partition the mean vector and covariance/precision matrices with respect to the $g$-th group of input variables as follows
\begin{align*}
\bmu = \beginmat{cc} \bmu_g \\ \bmu_{-g} \endmat, \quad \quad \bm{\Omega} = \beginmat{cc} \bm{\Omega}_{g} & \bm{\Omega}^{*\T}_{-g}\\ \bm{\Omega}^{*}_{-g} & \bm{\Omega}_{-g}\endmat, \quad \quad \bm{\Lambda} = \beginmat{cc} \bm{\Lambda}_{g} & \bm{\Lambda}^{*\T}_{-g}\\ \bm{\Lambda}^{*}_{-g} & \bm{\Lambda}_{-g}\endmat.
\end{align*}
where $\bm{\Omega}^{*}_{-g}$ and $\bm{\Lambda}^{*}_{-g}$ are used to denote the covariance and precision matrices between variables inside and outside of the set $\mathcal{S}_g$, respectively. Following the same logic used to derive Equation \eqref{eq:kld-j}, the RATE criterion to assess the centrality of group $g$ is given as
\begin{equation} \label{eq:kld-J}
\text{KLD}_g = \frac{1}{2}\bigg[\tr(\bOmega_{-g}\bLambda_{-g})-\log\left|\bOmega_{-g}\bLambda_{-g}\right|-\left(p-|\mathcal{S}_g|\right)+\bmu_g^{\T}\bDelta_g\bmu_g\bigg] \,,
\end{equation} 
where $|\mathcal{S}_g|$ is used to denote the cardinality of the $g$-th group, and $\bDelta_j = \bm{\Lambda}^{*\T}_{-g}\bm{\Lambda}_{-g}^{-1}\bm{\Lambda}^{*}_{-g}$ and characterizes the implied linear rate of change of information when the effect of all predictors in the $g$-th group are absent from the model. Throughout this section, we will refer to Equation \eqref{eq:kld-J} as the groupRATE criterion.

\subsection{Assessing Gene Importance in Genome-wide Association Studies} 

To demonstrate the groupRATE criterion, we turn to a genome-wide association (GWA) study of a heterogeneous stock of mice dataset from the Wellcome Trust Centre for Human Genetics \citep[\url{http://mtweb.cs.ucl.ac.uk/mus/www/mouse/index.shtml}]{Valdar:2006aa}. We focus on analyzing two quantitative traits: body mass index (BMI) and high-density lipoprotein (HDL) content. This dataset contains $n\approx$ 2000 and $p\approx$ 10000 single nucleotide polymorphisms (SNPs) with minor allele frequencies above 5\% --- with exact numbers varying slightly depending on the phenotype. In the traditional genome-wide association (GWA) framework, SNPs are individually tested for their marginal importance; however, this approach has been shown to have drawbacks and can suffer from low power when the architecture of a trait is complex \citep{manolio2009finding,yang2010common,Visscher:2012aa,Yang:2014aa}. As a result, recent approaches have aimed to combine SNPs within a chromosomal region to detect more biologically relevant genes and enriched pathways \citep{liu2010versatile,Ionita-Laza:2013aa,Nakka:2016aa,Zhu:2018aa,Cheng:2019aa}. Our interpretable Bayesian neural network framework can be used for similar tasks using groupRATE. 

Here, we use the Mouse Genome Database (MGD) \citep[\url{http://www.informatics.jax.org}]{Blake:2003aa} and define groups as collections of SNPs with genomic positions that fall within the same gene (or pseudogene). For simplicity, we eliminate genes with completely overlapping annotations. This resulted in 3,749 total genes (or groups of SNPs) across the 20 chromosomes in the mouse genome to be analyzed. After having trained our neural network, we run groupRATE on each of these groups using Equation \eqref{eq:kld-J} to create gene importance scores. To further validate the contextual relevance of our results, we use the enrichment analysis tool Enrichr \citep{Chen2013} to identify categories in the database of Genotypes and Phenotypes (dbGaP) with an overrepresentation of the significant genes reported by groupRATE within each trait. As a baseline comparison, we perform this same set of analyses using: \textit{(i)} the absolute value of coefficients from a group lasso mimic model \citep{Yuan06modelselection,friedman2010glasso} and \textit{(ii)}  the group importance scores derived from a random forest mimic model \citep{GREGORUTTI201515}.

\begin{figure}[htb!]
\centering
\includegraphics[width=\columnwidth]{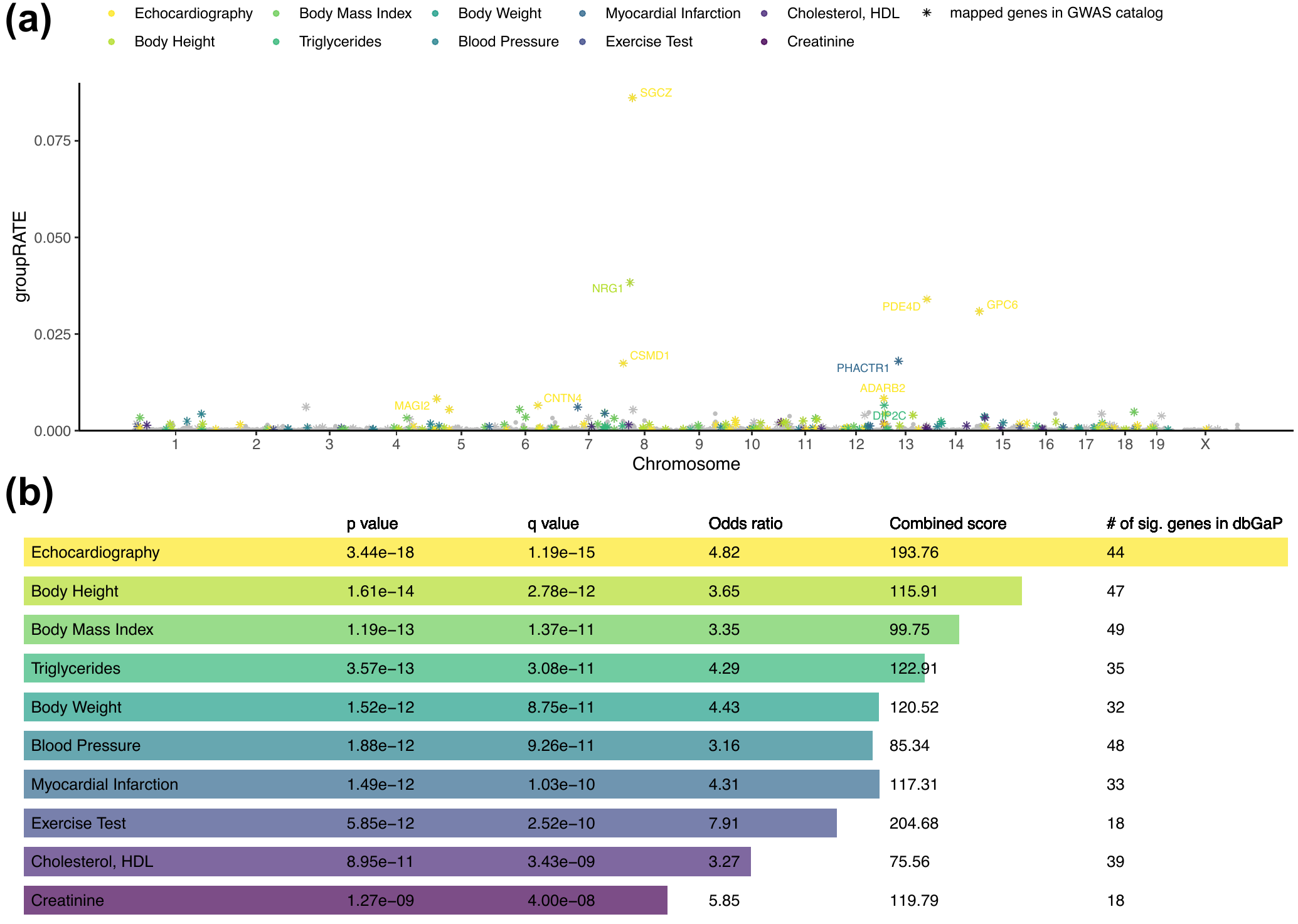}
\caption{Gene-level association results from applying groupRATE to body mass index (BMI) in the heterogeneous stock of mice dataset. Panel (a) depicts the group relative centrality measure for each gene plotted against their genomic positions. We annotate significant genes identified by the groupRATE (according to genome-wide threshold set to $T = $1/3749 genes = $2.67\times10^{-4}$) that overlap with those validated in the database of Genotypes and Phenotypes (dbGaP). In panel (b), we conduct gene set enrichment analysis using Enrichr \citep{Chen2013} to identify dbGaP categories enriched for significant gene-level associations reported by groupRATE. We highlight categories with $Q$-values (i.e., false discovery rates) less than 0.05 and annotate corresponding genes in the Manhattan plot.}
    \label{fig:bmi}
\end{figure}

\begin{figure}[htb!]
\centering
\includegraphics[width=\columnwidth]{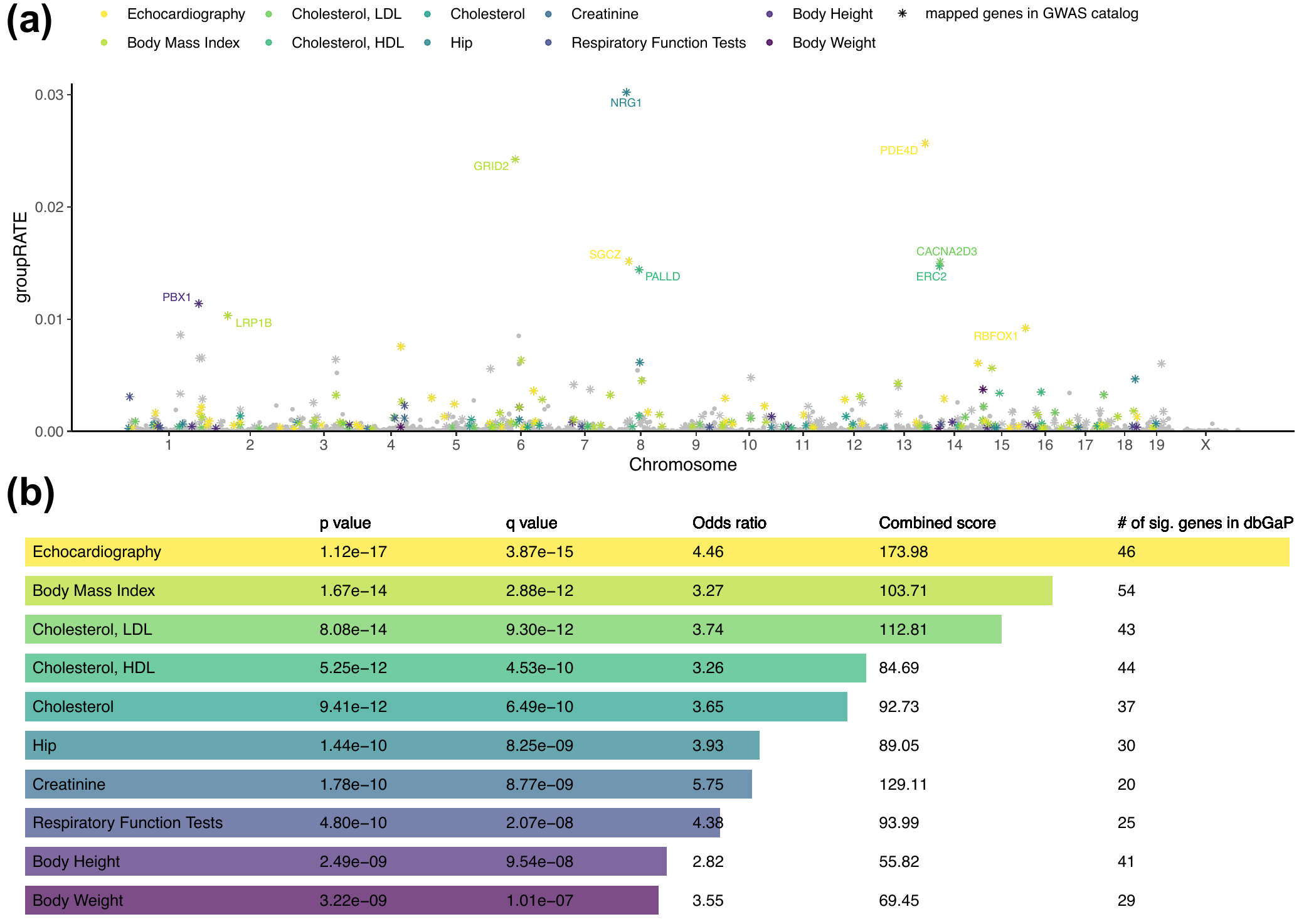}
\caption{Gene-level association results from applying groupRATE to high-density lipoprotein (HDL) content in the heterogeneous stock of mice dataset. Panel (a) depicts the group relative centrality measure for each gene plotted against their genomic positions. We annotate significant genes identified by the groupRATE (according to genome-wide threshold set to $T = $1/3749 genes = $2.67\times10^{-4}$) that overlap with those validated in the database of Genotypes and Phenotypes (dbGaP). In panel (b), we conduct gene set enrichment analysis using Enrichr \citep{Chen2013} to identify dbGaP categories enriched for significant gene-level associations reported by groupRATE. We highlight categories with $Q$-values (i.e., false discovery rates) less than 0.05 and annotate corresponding genes in the Manhattan plot.} 
    \label{fig:hdl}
\end{figure}

Overall, a large number of the significant genes identified by groupRATE have previously been annotated as having trait-specific associations in dbGaP (with genome-wide significance threshold set to $T = $1/3749 genes = $2.67\times10^{-4}$) (see Figures \ref{fig:bmi} and \ref{fig:hdl}). Previous computational studies have also shown many of these genes to have additive effects or nonlinear interaction effects that influence mice body composition. In this particular analysis, we attribute the selection of these genes to the nonlinear properties of the neural network and its ability to effectively learn complex patterns in data. For example, for BMI, the most important gene identified by groupRATE was \textit{Sgcz}, which is known to be involved with the dystrophin-associated glycoprotein complex \citep{Levy:2015aa} and has been suggested to have an ortholog that influences body mass in human beings \citep{Wang:2017aa}. Our approach also identified \textit{Nrg1} for both BMI and HDL, which is a gene that encodes a membrane glycoprotein that mediates cell-cell signaling and plays a critical role in the growth and development of multiple organ systems. Relevantly, the neuregulin gene family has been shown to associated with various aspects of metabolic health \citep{Wang:2014aa,Zhang:2018aa,Comas:2019aa}. While using Enrichr with these significant genes, the top categories with $Q$-values (i.e., false discovery rate) smaller than 0.05 for BMI and HDL included ``Body Mass Index'' and ``Cholesterol, HDL'', respectively, as well as other gene sets with verified and clinically relevant connections to each of the traits (e.g., ``Body Weight'' and ``Body Height''). The group lasso and group random forest mimic models were not as successful (see figures in Supplementary Material). For example, the group lasso failed to identify any biologically relevant enriched sets of genes in both BMI and HDL (again with threshold $T = $1/3749 genes = $2.67\times10^{-4}$ for consistency). In general, groupRATE was able to identify more significant genes associated with BMI and HDL (49 and 44, respectively) than the group lasso mimic model (11 and 5, respectively) and the random forest mimic model (33 and 24, respectively).


\section{Discussion} \label{sec:discussion}

In this paper, we developed a novel global interpretability method for deep neural networks. Here, we focused on settings in which predictor variables are intrinsically meaningful and the goal is to rank these features based on their scientific relevance. We worked in a very flexible variational Bayes approach to deep learning and proposed a sample covariance operator to develop an effect size analogue for the input variables of a neural network. Next, we extended the recently proposed RelATive cEntrality (RATE) measure \citep{crawford2018variable} to our setting, provided closed-form solutions for its implementation, and developed the groupRATE criterion for estimating the importance of groups of variables. Lastly, we illustrated the performance of our framework in broad applications including computer vision, natural language processing, and statistical genetics. Our method outperforms or achieves performance on par with the state-of-the-art, while avoiding the need for a separate and (often) time consuming tuning step.  

In its current form, we have focused on demonstrating the utility of RATE and groupRATE with a particular Bayesian neural network where only the weights on the outer layer are considered as random variables (see again Figure \ref{fig:bnn-example-architecture}). Note, however, that we are not restricted to this architecture and each of the innovations we have presented can be applied to any deep learning method that provides a notion of uncertainty over the predictions. The effect size analogue is merely a multivariate summary statistic which can be derived after fitting any model. This means that, as long as one has access to empirical estimates of its posterior distribution, relative centrality measures can always be computed. While the variational Bayes framework described in our work gives an exact Gaussian posterior over $\f$, many recent works have focused on calculating approximations to the posterior of an already-trained deterministic network using Laplace approximations \citep{ritter2018scalable} or stochastic gradient descent iterates \citep{maddox2019simple}. Combining these approaches with RATE would allow variable importance calculations to be performed on an already trained deterministic network (without the need for retraining with a mean-field variational posterior on the final layer).

Motivated by these results, there are several interesting future directions that remain. For example, in the current study, we strictly focus on interpreting the significance of variables at the input layer of DNNs. However, given the network architecture that we consider, it is also possible to examine the importance of hidden layers using RATE. Essentially, if we impose interpretations onto these layers in the context of some application, then we may use the centrality measure to assess how the corresponding nodes (i.e., specific groups of input variables) contribute to predictive accuracy. One example of this would be to construct a partially connected network architecture based on the literature and hierarchical nature of biological enrichment analyses in genome-wide association studies.





One of the main limitations of RATE is the $\mathcal{O}(p^{4})$ cost of solving the linear system in $\delta_j = \bm{\lambda}^{\T}_{-j}\bm{\Lambda}_{-j}^{-1}\bm{\lambda}_{-j}$ in Equation \eq{eq:kld-j-approx}, which is the computational bottleneck given $p$ variables in the model. This calculation is made up of $p$ independent $\mathcal{O}(p^3)$ operations. If $p\approx10^2$, our analyses remain feasible, assuming that the calculations can be parallelized over some number of computing cores. However, as the dimension $p$ outstrips the available number of cores and the cost reverts to $\mathcal{O}(p^4)$. We found it empirically difficult to implement RATE on datasets with more than $p\approx10^4$ features. Unfortunately, this precludes us from applying RATE to application such as radiomics where the number of pixels can be around $p\approx10^6$ for high-quality medical images  \citep{ukbiobankbrainmri}. Radiomics would otherwise be an attractive application for RATE, as pixels have a fixed meaning across examples due to the fact that the medical images are almost always well-aligned with one another. It is possible to reduce the cost of calculating $\delta_j = \bm{\lambda}^{\T}_{-j}\bm{\Lambda}_{-j}^{-1}\bm{\lambda}_{-j}$ to $\mathcal{O}(p^2)$ using the Sherman-Morris formula. We can do this by noticing that the difference between $\bm{\Lambda}_{-j}^{-1}$ for any two values of $j$ is a rank-two matrix, since any $\bm{\Lambda}_{-j}^{-1}$ is formed by removing the $j$-th row and column from the precision matrix $\bm{\Lambda}$. Given an inverse $\bm{\Lambda}_{-j}^{-1}$, we can therefore calculate $\bm{\Lambda}_{-{(j+1)}}^{-1}$ using the Sherman-Morris formula to update $\bm{\Lambda}_{-j}^{-1}$ rather than calculating it from scratch \citep{hager1989updating}. Alternatively, we can use a similar argument to update the solution to $\bm{\Lambda}_{-j}^{-1}\bm{\lambda}_{-j}$ \citep{hammarling2008updating}. We will explore this implementation in future work.


\section{Software Availability}

Software for implementing the interpretable Bayesian neural network framework with RATE significance measures is carried out in R and Python code, which is available at \url{https://github.com/lorinanthony/RATE}.

\section{Acknowledgments}

This research was supported by grants P20GM109035 (COBRE Center for Computational Biology of Human Disease; PI Rand) and P20GM103645 (COBRE Center for Central Nervous; PI Sanes) from the NIH NIGMS, 2U10CA180794-06 from the NIH NCI and the Dana Farber Cancer Institute (PIs Gray and Gatsonis), as well as by an Alfred P. Sloan Research Fellowship awarded to Lorin Crawford. Sarah Filippi is also partially supported by the EPSRC (grant EP/R013519/1) and Jonathan Ish-Horowicz gratefully acknowledges funding from the Wellcome Trust (PhD studentship 215359/Z/19/Z). Any opinions, findings, and conclusions or recommendations expressed in this material are those of the author(s) and do not necessarily reflect the views of any of the funders. 

\bibliography{bibliography}

\begin{thebibliography}{72}
\providecommand{\natexlab}[1]{#1}
\providecommand{\url}[1]{\texttt{#1}}
\expandafter\ifx\csname urlstyle\endcsname\relax
  \providecommand{\doi}[1]{doi: #1}\else
  \providecommand{\doi}{doi: \begingroup \urlstyle{rm}\Url}\fi

\bibitem[Adebayo et~al.(2018)Adebayo, Gilmer, Muelly, Goodfellow, Hardt, and
  Kim]{adebayo2018sanity}
Julius Adebayo, Justin Gilmer, Michael Muelly, Ian Goodfellow, Moritz Hardt,
  and Been Kim.
\newblock Sanity checks for saliency maps.
\newblock In \emph{Advances in Neural Information Processing Systems}, pages
  9505--9515, 2018.

\bibitem[Ancona et~al.(2018)Ancona, Ceolini, {\"O}ztireli, and
  Gross]{ancona2017towards}
Marco Ancona, Enea Ceolini, Cengiz {\"O}ztireli, and Markus Gross.
\newblock Towards better understanding of gradient-based attribution methods
  for deep neural networks.
\newblock In \emph{6th International Conference on Learning Representations,
  ICLR 2018-Conference Track Proceedings}, volume~6. International Conference
  on Representation Learning, 2018.

\bibitem[Arya et~al.(2019)Arya, Bellamy, Chen, Dhurandhar, Hind, Hoffman,
  Houde, Liao, Luss, Mojsilovi{\'c}, et~al.]{arya2019one}
Vijay Arya, Rachel~KE Bellamy, Pin-Yu Chen, Amit Dhurandhar, Michael Hind,
  Samuel~C Hoffman, Stephanie Houde, Q~Vera Liao, Ronny Luss, Aleksandra
  Mojsilovi{\'c}, et~al.
\newblock One explanation does not fit all: {A} toolkit and taxonomy of {AI}
  explainability techniques.
\newblock \emph{arXiv preprint arXiv:1909.03012}, 2019.

\bibitem[Ba and Caruana(2014)]{ba2014deep}
Jimmy Ba and Rich Caruana.
\newblock Do deep nets really need to be deep?
\newblock In \emph{Advances in Neural Information Processing Systems}, pages
  2654--2662, 2014.

\bibitem[Bach et~al.(2015)Bach, Binder, Montavon, Klauschen, M{\"u}ller, and
  Samek]{bach2015pixel}
Sebastian Bach, Alexander Binder, Gr{\'e}goire Montavon, Frederick Klauschen,
  Klaus-Robert M{\"u}ller, and Wojciech Samek.
\newblock On pixel-wise explanations for non-linear classifier decisions by
  layer-wise relevance propagation.
\newblock \emph{PLOS One}, 10\penalty0 (7):\penalty0 e0130140, 2015.

\bibitem[Barber and Bishop(1998)]{barber1998ensemble}
David Barber and Christopher~M Bishop.
\newblock Ensemble learning in {B}ayesian neural networks.
\newblock \emph{NATO ASI Series F Computer and Systems Sciences}, 168:\penalty0
  215--238, 1998.

\bibitem[Baumeister et~al.(2001)Baumeister, Bratslavsky, Finkenauer, and
  Vohs]{baumeister2001bad}
Roy~F Baumeister, Ellen Bratslavsky, Catrin Finkenauer, and Kathleen~D Vohs.
\newblock Bad is stronger than good.
\newblock \emph{Review of General Psychology}, 5\penalty0 (4):\penalty0 323,
  2001.

\bibitem[Blake et~al.(2003)Blake, Richardson, Bult, Kadin, Eppig, and
  Group]{Blake:2003aa}
Judith~A Blake, Joel~E Richardson, Carol~J Bult, Jim~A Kadin, Janan~T Eppig,
  and Mouse Genome~Database Group.
\newblock {MGD}: {T}he mouse genome database.
\newblock \emph{Nucleic Acids Research}, 31\penalty0 (1):\penalty0 193--195,
  2003.

\bibitem[Blei et~al.(2017)Blei, Kucukelbir, and McAuliffe]{blei2017variational}
David~M Blei, Alp Kucukelbir, and Jon~D McAuliffe.
\newblock Variational inference: {A} review for statisticians.
\newblock \emph{Journal of the American statistical Association}, 112\penalty0
  (518):\penalty0 859--877, 2017.

\bibitem[Breiman(2001)]{breiman2001random}
Leo Breiman.
\newblock Random forests.
\newblock \emph{Machine {L}earning}, 45\penalty0 (1):\penalty0 5--32, 2001.

\bibitem[Carvalho et~al.(2019)Carvalho, Pereira, and
  Cardoso]{carvalho2019machine}
Diogo~V Carvalho, Eduardo~M Pereira, and Jaime~S Cardoso.
\newblock Machine learning interpretability: A survey on methods and metrics.
\newblock \emph{Electronics}, 8\penalty0 (8):\penalty0 832, 2019.

\bibitem[Che et~al.(2016)Che, Purushotham, Khemani, and
  Liu]{che2016interpretable}
Zhengping Che, Sanjay Purushotham, Robinder Khemani, and Yan Liu.
\newblock Interpretable deep models for {ICU} outcome prediction.
\newblock In \emph{AMIA Annual Symposium Proceedings}, volume 2016, page 371.
  American Medical Informatics Association, 2016.

\bibitem[Chen et~al.(2013)Chen, Tan, Kou, Duan, Wang, Meirelles, Clark, and
  Ma'ayan]{Chen2013}
Edward~Y. Chen, Christopher~M. Tan, Yan Kou, Qiaonan Duan, Zichen Wang,
  Gabriela~Vaz Meirelles, Neil~R. Clark, and Avi Ma'ayan.
\newblock Enrichr: {I}nteractive and collaborative html5 gene list enrichment
  analysis tool.
\newblock \emph{BMC Bioinformatics}, 14\penalty0 (1):\penalty0 128, 2013.
\newblock \doi{10.1186/1471-2105-14-128}.

\bibitem[Chen et~al.(2007)Chen, Liu, Zhang, and Zhang]{chen2007forest}
Xiang Chen, Ching-Ti Liu, Meizhuo Zhang, and Heping Zhang.
\newblock A forest-based approach to identifying gene and gene--gene
  interactions.
\newblock \emph{Proceedings of the National Academy of Sciences}, 104\penalty0
  (49):\penalty0 19199--19203, 2007.

\bibitem[Cheng et~al.(2019)Cheng, Ramachandran, and Crawford]{Cheng:2019aa}
Wei Cheng, Sohini Ramachandran, and Lorin Crawford.
\newblock Epsilon-genic effects bridge the gap between polygenic and omnigenic
  complex traits.
\newblock \emph{bioRxiv preprint 597484}, 2019.

\bibitem[Clough et~al.(2019)Clough, Oksuz, Puyol-Ant{\'o}n, Ruijsink, King, and
  Schnabel]{clough2019global}
James~R Clough, Ilkay Oksuz, Esther Puyol-Ant{\'o}n, Bram Ruijsink, Andrew~P
  King, and Julia~A Schnabel.
\newblock Global and local interpretability for cardiac {MRI} classification.
\newblock In \emph{International Conference on Medical Image Computing and
  Computer-Assisted Intervention}, pages 656--664. Springer, 2019.

\bibitem[Comas et~al.(2019)Comas, Mart{\'\i}nez, Sabater, Ortega, Latorre,
  D{\'\i}az-S{\'a}ez, Aragon{\'e}s, Camps, Gum{\`a}, Ricart,
  Fern{\'a}ndez-Real, and Moreno-Navarrete]{Comas:2019aa}
Ferran Comas, Cristina Mart{\'\i}nez, M{\`o}nica Sabater, Francisco Ortega,
  Jessica Latorre, Francisco D{\'\i}az-S{\'a}ez, Julian Aragon{\'e}s, Marta
  Camps, Anna Gum{\`a}, Wifredo Ricart, Jos{\'e}Manuel Fern{\'a}ndez-Real, and
  Jos{\'e}Mar{\'\i}a Moreno-Navarrete.
\newblock Neuregulin 4 is a novel marker of beige adipocyte precursor cells in
  human adipose tissue.
\newblock \emph{Frontiers in Physiology}, 10:\penalty0 39--39, 2019.

\bibitem[Cotter et~al.(2011)Cotter, Keshet, and Srebro]{Cotter:2011aa}
Andrew Cotter, Joseph Keshet, and Nathan Srebro.
\newblock Explicit approximations of the {G}aussian kernel.
\newblock \emph{arXiv preprint arXiv:1109.4603}, 2011.

\bibitem[Crawford et~al.(2018)Crawford, Wood, Zhou, and
  Mukherjee]{crawford2018bayesian}
Lorin Crawford, Kris~C Wood, Xiang Zhou, and Sayan Mukherjee.
\newblock Bayesian approximate kernel regression with variable selection.
\newblock \emph{Journal of the American Statistical Association}, 113\penalty0
  (524):\penalty0 1710--1721, 2018.

\bibitem[Crawford et~al.(2019)Crawford, Flaxman, Runcie, and
  West]{crawford2018variable}
Lorin Crawford, Seth~R Flaxman, Daniel~E Runcie, and Mike West.
\newblock Variable prioritization in nonlinear black box methods: {A} genetic
  association case study.
\newblock \emph{Annals of Applied Statistics}, 13\penalty0 (2):\penalty0
  958--989, 2019.

\bibitem[Doshi-Velez and Kim(2017)]{doshi2017towards}
Finale Doshi-Velez and Been Kim.
\newblock Towards a rigorous science of interpretable machine learning.
\newblock \emph{arXiv preprint arXiv:1702.08608}, 2017.

\bibitem[Erhan et~al.(2009)Erhan, Bengio, Courville, and
  Vincent]{erhan2009visualizing}
Dumitru Erhan, Yoshua Bengio, Aaron Courville, and Pascal Vincent.
\newblock Visualizing higher-layer features of a deep network.
\newblock \emph{University of Montreal}, 1341\penalty0 (3):\penalty0 1, 2009.

\bibitem[Friedman(2001)]{friedman2001greedy}
Jerome~H Friedman.
\newblock Greedy function approximation: {A} gradient boosting machine.
\newblock \emph{Annals of Statistics}, 29\penalty0 (5):\penalty0 1189--1232,
  2001.

\bibitem[Friedman et~al.(2010)Friedman, Hastie, and
  Tibshirani]{friedman2010glasso}
Jerome~H Friedman, Trevor Hastie, and Robert Tibshirani.
\newblock Greedy function approximation: {A} gradient boosting machine.
\newblock \emph{arXiv preprint arXiv:1001.0736}, 2010.

\bibitem[Frosst and Hinton(2017)]{frosst2017distilling}
Nicholas Frosst and Geoffrey Hinton.
\newblock Distilling a neural network into a soft decision tree.
\newblock \emph{arXiv preprint arXiv:1711.09784}, 2017.

\bibitem[Ghorbani et~al.(2019)Ghorbani, Abid, and
  Zou]{ghorbani2019interpretation}
Amirata Ghorbani, Abubakar Abid, and James Zou.
\newblock Interpretation of neural networks is fragile.
\newblock In \emph{Proceedings of the AAAI Conference on Artificial
  Intelligence}, volume~33, pages 3681--3688, 2019.

\bibitem[Graves(2011)]{graves2011practical}
Alex Graves.
\newblock Practical variational inference for neural networks.
\newblock In \emph{Advances in Neural Information Processing Systems}, pages
  2348--2356, 2011.

\bibitem[Gregorutti et~al.(2015)Gregorutti, Michel, and
  Saint-Pierre]{GREGORUTTI201515}
Baptiste Gregorutti, Bertrand Michel, and Philippe Saint-Pierre.
\newblock Grouped variable importance with random forests and application to
  multiple functional data analysis.
\newblock \emph{Computational Statistics \& Data Analysis}, 90:\penalty0
  15--35, 2015.

\bibitem[Guidotti et~al.(2018)Guidotti, Monreale, Ruggieri, Turini, Giannotti,
  and Pedreschi]{guidotti2018survey}
Riccardo Guidotti, Anna Monreale, Salvatore Ruggieri, Franco Turini, Fosca
  Giannotti, and Dino Pedreschi.
\newblock A survey of methods for explaining black box models.
\newblock \emph{ACM Computing Surveys (CSUR)}, 51\penalty0 (5):\penalty0 93,
  2018.

\bibitem[Hager(1989)]{hager1989updating}
William~W Hager.
\newblock Updating the inverse of a matrix.
\newblock \emph{SIAM review}, 31\penalty0 (2):\penalty0 221--239, 1989.

\bibitem[Hall(2019)]{hall2019guidelines}
Patrick Hall.
\newblock Guidelines for responsible and human-centered use of explainable
  machine learning.
\newblock \emph{arXiv preprint arXiv:1906.03533}, 2019.

\bibitem[Hammarling and Lucas(2008)]{hammarling2008updating}
Sven Hammarling and Craig Lucas.
\newblock Updating the {QR} factorization and the least squares problem.
\newblock 2008.

\bibitem[Hinton et~al.(2015)Hinton, Vinyals, and Dean]{hinton2015distilling}
Geoffrey Hinton, Oriol Vinyals, and Jeff Dean.
\newblock Distilling the knowledge in a neural network.
\newblock \emph{arXiv preprint arXiv:1503.02531}, 2015.

\bibitem[Hinton and Van~Camp(1993)]{hinton1993keeping}
Geoffrey~E Hinton and Drew Van~Camp.
\newblock Keeping neural networks simple by minimizing the description length
  of the weights.
\newblock In \emph{Proceedings of the Sixth Annual Conference on Computational
  Learning Theory}, pages 5--13. ACM, 1993.

\bibitem[Ionita-Laza et~al.(2013)Ionita-Laza, Lee, Makarov, Buxbaum, and
  Lin]{Ionita-Laza:2013aa}
Iuliana Ionita-Laza, Seunggeun Lee, Vlad Makarov, Joseph~D. Buxbaum, and Xihong
  Lin.
\newblock Sequence kernel association tests for the combined effect of rare and
  common variants.
\newblock \emph{American Journal of Human Genetics}, 92\penalty0 (6):\penalty0
  841--853, 2013.
\newblock \doi{https://doi.org/10.1016/j.ajhg.2013.04.015}.

\bibitem[Jiang and Reif(2015)]{Jiang:2015aa}
Yong Jiang and Jochen~C. Reif.
\newblock Modeling epistasis in genomic selection.
\newblock \emph{Genetics}, 201:\penalty0 759--768, 2015.

\bibitem[Kindermans et~al.(2019)Kindermans, Hooker, Adebayo, Alber, Sch{\"u}tt,
  D{\"a}hne, Erhan, and Kim]{kindermans2019reliability}
Pieter-Jan Kindermans, Sara Hooker, Julius Adebayo, Maximilian Alber, Kristof~T
  Sch{\"u}tt, Sven D{\"a}hne, Dumitru Erhan, and Been Kim.
\newblock The (un) reliability of saliency methods.
\newblock In \emph{Explainable AI: Interpreting, Explaining and Visualizing
  Deep Learning}, pages 267--280. Springer, 2019.

\bibitem[Kingma and Ba(2014)]{kingma2014adam}
Diederik~P Kingma and Jimmy Ba.
\newblock Adam: A method for stochastic optimization.
\newblock \emph{arXiv preprint arXiv:1412.6980}, 2014.

\bibitem[Kingma et~al.(2015)Kingma, Salimans, and
  Welling]{kingma2015variational}
Durk~P Kingma, Tim Salimans, and Max Welling.
\newblock Variational dropout and the local reparameterization trick.
\newblock In \emph{Advances in Neural Information Processing Systems}, pages
  2575--2583, 2015.

\bibitem[Kolmogorov and Rozanov(1960)]{Kolmogorov:1960aa}
Andrei~Nikolaevich Kolmogorov and Yu~A Rozanov.
\newblock On strong mixing conditions for stationary {Gaussian} processes.
\newblock \emph{Theory Probability and Its Applications}, 5\penalty0
  (2):\penalty0 204--208, 1960.

\bibitem[Kuleshov et~al.(2016)Kuleshov, Jones, Rouillard, Fernandez, Duan,
  Wang, Koplev, Jenkins, Jagodnik, Lachmann, McDermott, Monteiro, Gundersen,
  and Ma'ayan]{10.1093/nar/gkw377}
Maxim~V. Kuleshov, Matthew~R. Jones, Andrew~D. Rouillard, Nicolas~F. Fernandez,
  Qiaonan Duan, Zichen Wang, Simon Koplev, Sherry~L. Jenkins, Kathleen~M.
  Jagodnik, Alexander Lachmann, Michael~G. McDermott, Caroline~D. Monteiro,
  Gregory~W. Gundersen, and Avi Ma'ayan.
\newblock {Enrichr: a comprehensive gene set enrichment analysis web server
  2016 update}.
\newblock \emph{Nucleic Acids Research}, 44\penalty0 (W1):\penalty0 W90--W97,
  2016.

\bibitem[Kuttichira et~al.(2019)Kuttichira, Gupta, Li, Rana, and
  Venkatesh]{kuttichira2019explaining}
Deepthi~Praveenlal Kuttichira, Sunil Gupta, Cheng Li, Santu Rana, and Svetha
  Venkatesh.
\newblock Explaining black-box models using interpretable surrogates.
\newblock In \emph{Pacific Rim International Conference on Artificial
  Intelligence}, pages 3--15. Springer, 2019.

\bibitem[LeCun(1998)]{lecun1998mnist}
Yann LeCun.
\newblock The {MNIST} {D}atabase of {H}andwritten {D}igits.
\newblock \emph{http://yann. lecun. com/exdb/mnist/}, 1998.

\bibitem[LeCun et~al.(2015)LeCun, Bengio, and Hinton]{lecun2015deep}
Yann LeCun, Yoshua Bengio, and Geoffrey Hinton.
\newblock Deep learning.
\newblock \emph{Nature}, 521\penalty0 (7553):\penalty0 436--444, 2015.

\bibitem[Levy et~al.(2015)Levy, Mott, Iraqi, and Gabet]{Levy:2015aa}
Roei Levy, Richard~F. Mott, Fuad~A. Iraqi, and Yankel Gabet.
\newblock Collaborative cross mice in a genetic association study reveal new
  candidate genes for bone microarchitecture.
\newblock \emph{BMC Genomics}, 16\penalty0 (1):\penalty0 1013, 2015.
\newblock \doi{10.1186/s12864-015-2213-x}.

\bibitem[Liu et~al.(2010)Liu, Mcrae, Nyholt, Medland, Wray, Brown, Hayward,
  Montgomery, Visscher, Martin, et~al.]{liu2010versatile}
Jimmy~Z Liu, Allan~F Mcrae, Dale~R Nyholt, Sarah~E Medland, Naomi~R Wray,
  Kevin~M Brown, Nicholas~K Hayward, Grant~W Montgomery, Peter~M Visscher,
  Nicholas~G Martin, et~al.
\newblock A versatile gene-based test for genome-wide association studies.
\newblock \emph{American Journal of Human Genetics}, 87\penalty0 (1):\penalty0
  139--145, 2010.

\bibitem[Lundervold and Lundervold(2019)]{lundervold2019overview}
Alexander~Selvikv{\aa}g Lundervold and Arvid Lundervold.
\newblock An overview of deep learning in medical imaging focusing on mri.
\newblock \emph{Zeitschrift f{\"u}r Medizinische Physik}, 29\penalty0
  (2):\penalty0 102--127, 2019.

\bibitem[Maas et~al.(2011)Maas, Daly, Pham, Huang, Ng, and
  Potts]{maas2011learning}
Andrew~L Maas, Raymond~E Daly, Peter~T Pham, Dan Huang, Andrew~Y Ng, and
  Christopher Potts.
\newblock Learning word vectors for sentiment analysis.
\newblock In \emph{Proceedings of the 49th {A}nnual {M}eeting of the
  {A}ssociation for {C}omputational {L}inguistics ({ACL} 2011)}, pages
  142--150. Association for Computational Linguistics, 2011.

\bibitem[Maddox et~al.(2019)Maddox, Izmailov, Garipov, Vetrov, and
  Wilson]{maddox2019simple}
Wesley~J Maddox, Pavel Izmailov, Timur Garipov, Dmitry~P Vetrov, and
  Andrew~Gordon Wilson.
\newblock A simple baseline for {B}ayesian uncertainty in deep learning.
\newblock In \emph{Advances in Neural Information Processing Systems}, pages
  13132--13143, 2019.

\bibitem[Manolio et~al.(2009)Manolio, Collins, Cox, Goldstein, Hindorff,
  Hunter, McCarthy, Ramos, Cardon, Chakravarti, Cho, Guttmacher, Kong,
  Kruglyak, Mardis, Rotimi, Slatkin, Valle, Whittemore, Boehnke, Clark,
  Eichler, Gibson, Haines, Mackay, McCarroll, and Visscher]{manolio2009finding}
Teri~A Manolio, Francis~S Collins, Nancy~J Cox, David~B Goldstein, Lucia~A
  Hindorff, David~J Hunter, Mark~I McCarthy, Erin~M Ramos, Lon~R Cardon,
  Aravinda Chakravarti, Judy~H Cho, Alan~E Guttmacher, Augustine Kong, Leonid
  Kruglyak, Elaine Mardis, Charles~N Rotimi, Montgomery Slatkin, David Valle,
  Alice~S Whittemore, Michael Boehnke, Andrew~G Clark, Evan~E Eichler, Greg
  Gibson, Jonathan~L Haines, Trudy F~C Mackay, Steven~A McCarroll, and Peter~M
  Visscher.
\newblock Finding the missing heritability of complex diseases.
\newblock \emph{Nature}, 461\penalty0 (7265):\penalty0 747--753, 2009.
\newblock \doi{10.1038/nature08494}.

\bibitem[Nakka et~al.(2016)Nakka, Raphael, and Ramachandran]{Nakka:2016aa}
Priyanka Nakka, Benjamin~J. Raphael, and Sohini Ramachandran.
\newblock Gene and network analysis of common variants reveals novel
  associations in multiple complex diseases.
\newblock \emph{Genetics}, 204\penalty0 (2):\penalty0 783--798, 2016.

\bibitem[Pedregosa et~al.(2011)Pedregosa, Varoquaux, Gramfort, Michel, Thirion,
  Grisel, Blondel, Prettenhofer, Weiss, Dubourg, Vanderplas, Passos,
  Cournapeau, Brucher, Perrot, and Duchesnay]{scikit-learn}
F.~Pedregosa, G.~Varoquaux, A.~Gramfort, V.~Michel, B.~Thirion, O.~Grisel,
  M.~Blondel, P.~Prettenhofer, R.~Weiss, V.~Dubourg, J.~Vanderplas, A.~Passos,
  D.~Cournapeau, M.~Brucher, M.~Perrot, and E.~Duchesnay.
\newblock Scikit-learn: Machine learning in {P}ython.
\newblock \emph{Journal of Machine Learning Research}, 12:\penalty0 2825--2830,
  2011.

\bibitem[Rasmussen and Williams(2006)]{rasmussen2006gaussian}
Carl~Edward Rasmussen and Christopher~KI Williams.
\newblock Gaussian {P}rocesses for {M}achine {L}earning, 2006.

\bibitem[Ribeiro et~al.(2016)Ribeiro, Singh, and Guestrin]{ribeiro2016should}
Marco~Tulio Ribeiro, Sameer Singh, and Carlos Guestrin.
\newblock Why {S}hould {I} {T}rust {Y}ou?: {E}xplaining the {P}redictions of
  any {C}lassifier.
\newblock In \emph{Proceedings of the 22nd ACM SIGKDD International Conference
  on Knowledge Discovery and Data Mining}, pages 1135--1144. ACM, 2016.

\bibitem[Ritter et~al.(2018)Ritter, Botev, and Barber]{ritter2018scalable}
Hippolyt Ritter, Aleksandar Botev, and David Barber.
\newblock A scalable {L}aplace approximation for neural networks.
\newblock In \emph{6th International Conference on Learning Representations,
  ICLR 2018-Conference Track Proceedings}, volume~6. International Conference
  on Representation Learning, 2018.

\bibitem[Sch{\"o}lkopf et~al.(2001)Sch{\"o}lkopf, Herbrich, and
  Smola]{Scholkopf:2001aa}
Bernhard Sch{\"o}lkopf, Ralf Herbrich, and Alex~J. Smola.
\newblock A generalized representer theorem.
\newblock In \emph{Proceedings of the 14th Annual Conference on Computational
  Learning Theory and and 5th European Conference on Computational Learning
  Theory}, pages 416--426, London, UK, UK, 2001. Springer-Verlag.

\bibitem[Sch{\"o}lkopf et~al.(2002)Sch{\"o}lkopf, Smola, Bach,
  et~al.]{scholkopf2002learning}
Bernhard Sch{\"o}lkopf, Alexander~J Smola, Francis Bach, et~al.
\newblock \emph{Learning with {K}ernels: {S}upport {V}ector {M}achines,
  {R}egularization, {O}ptimization, and {B}eyond}.
\newblock MIT press, 2002.

\bibitem[Shrikumar et~al.(2016)Shrikumar, Greenside, Shcherbina, and
  Kundaje]{shrikumar2016not}
Avanti Shrikumar, Peyton Greenside, Anna Shcherbina, and Anshul Kundaje.
\newblock Not just a black box: {L}earning important features through
  propagating activation differences.
\newblock \emph{arXiv preprint arXiv:1605.01713}, 2016.

\bibitem[Smith et~al.(2020)Smith, Alfaro-Almagro, and
  Miller]{ukbiobankbrainmri}
Stephen~M. Smith, Fidel Alfaro-Almagro, and Karla~L. Miller.
\newblock {UK Biobank Brain Imaging Documentation Version 1.7}.
\newblock
  \url{https://biobank.ctsu.ox.ac.uk/crystal/crystal/docs/brain_mri.pdf}, 2020.
\newblock 17-04-2020.

\bibitem[Sundararajan et~al.(2017)Sundararajan, Taly, and
  Yan]{sundararajan2017axiomatic}
Mukund Sundararajan, Ankur Taly, and Qiqi Yan.
\newblock Axiomatic attribution for deep networks.
\newblock In \emph{Proceeding of the 34th International Conference on Machine
  Learning}, pages 3319--3328, 2017.

\bibitem[Valdar et~al.(2006)Valdar, Solberg, Gauguier, Burnett, Klenerman,
  Cookson, Taylor, Rawlins, Mott, and Flint]{Valdar:2006aa}
William Valdar, Leah~C Solberg, Dominique Gauguier, Stephanie Burnett, Paul
  Klenerman, William~O Cookson, Martin~S Taylor, J~Nicholas~P Rawlins, Richard
  Mott, and Jonathan Flint.
\newblock Genome-wide genetic association of complex traits in heterogeneous
  stock mice.
\newblock \emph{Nature Genetics}, 38\penalty0 (8):\penalty0 879--887, 2006.

\bibitem[Visscher et~al.(2012)Visscher, Brown, McCarthy, and
  Yang]{Visscher:2012aa}
Peter~M. Visscher, Matthew~A. Brown, Mark~I. McCarthy, and Jian Yang.
\newblock Five years of gwas discovery.
\newblock \emph{American Journal of Human Genetics}, 90\penalty0 (1):\penalty0
  7--24, 2012.

\bibitem[Wang and Rudin(2015)]{wang2015falling}
Fulton Wang and Cynthia Rudin.
\newblock Falling rule lists.
\newblock In \emph{Artificial Intelligence and Statistics}, pages 1013--1022,
  2015.

\bibitem[Wang et~al.(2014)Wang, Zhao, Meng, Kern, Dietrich, Chen, Cozacov,
  Zhou, Okunade, Su, Li, Bl{\"u}her, and Lin]{Wang:2014aa}
Guo-Xiao Wang, Xu-Yun Zhao, Zhuo-Xian Meng, Matthias Kern, Arne Dietrich,
  Zhimin Chen, Zoharit Cozacov, Dequan Zhou, Adewole~L Okunade, Xiong Su,
  Siming Li, Matthias Bl{\"u}her, and Jiandie~D Lin.
\newblock The brown fat-enriched secreted factor {Nrg4} preserves metabolic
  homeostasis through attenuation of hepatic lipogenesis.
\newblock \emph{Nature Medicine}, 20\penalty0 (12):\penalty0 1436--1443, 2014.

\bibitem[Wang et~al.(2017)Wang, Moon, Wu, Amos, Hung, Tardon, Andrew, Chen,
  Christiani, Albanes, Heijden, Duell, Rennert, Goodman, Liu, Mckay, Yuan,
  Field, Manjer, Grankvist, Kiemeney, Marchand, Teare, Schabath, Johansson,
  Aldrich, Davies, Johansson, Tsao, Caporaso, Lazarus, Lam, Bojesen, Arnold,
  Wu, Zong, Hong, and Ho]{Wang:2017aa}
Tao Wang, Jee-Young Moon, Yiqun Wu, Christopher~I Amos, Rayjean~J Hung, Adonina
  Tardon, Angeline Andrew, Chu Chen, David~C Christiani, Demetrios Albanes,
  Erik H F M van~der Heijden, Eric Duell, Gadi Rennert, Gary Goodman, Geoffrey
  Liu, James~D Mckay, Jian-Min Yuan, John~K Field, Jonas Manjer, Kjell
  Grankvist, Lambertus~A Kiemeney, Loic~Le Marchand, M~Dawn Teare, Matthew~B
  Schabath, Mattias Johansson, Melinda~C Aldrich, Michael Davies, Mikael
  Johansson, Ming-Sound Tsao, Neil Caporaso, Philip Lazarus, Stephen Lam,
  Stig~E Bojesen, Susanne Arnold, Xifeng Wu, Xuchen Zong, Yun-Chul Hong, and
  Gloria Y~F Ho.
\newblock Pleiotropy of genetic variants on obesity and smoking phenotypes:
  Results from the oncoarray project of the international lung cancer
  consortium.
\newblock \emph{PLOS One}, 12\penalty0 (9):\penalty0 e0185660, 2017.

\bibitem[Wold et~al.(1984)Wold, Ruhe, Wold, and Dunn]{Wold:1984aa}
Svante Wold, Arnold Ruhe, Herman Wold, and WJ~Dunn, III.
\newblock The collinearity problem in linear regression. the partial least
  squares (pls) approach to generalized inverses.
\newblock \emph{SIAM Journal on Scientific and Statistical Computing},
  5\penalty0 (3):\penalty0 735--743, 1984.

\bibitem[Wu et~al.(2010)Wu, Kraft, Epstein, Taylor, Chanock, Hunter, and
  Lin]{Wu:2010aa}
Michael~C Wu, Peter Kraft, Michael~P Epstein, Deanne~M Taylor, Stephen~J
  Chanock, David~J Hunter, and Xihong Lin.
\newblock Powerful {SNP}-set analysis for case-control genome-wide association
  studies.
\newblock \emph{American Journal of Human Genetics}, 86\penalty0 (6):\penalty0
  929--942, 2010.

\bibitem[Yang et~al.(2010)Yang, Benyamin, McEvoy, Gordon, Henders, Nyholt,
  Madden, Heath, Martin, Montgomery, Goddard, and Visscher]{yang2010common}
Jian Yang, Beben Benyamin, Brian~P McEvoy, Scott Gordon, Anjali~K Henders,
  Dale~R Nyholt, Pamela~A Madden, Andrew~C Heath, Nicholas~G Martin, Grant~W
  Montgomery, Michael~E Goddard, and Peter~M Visscher.
\newblock Common snps explain a large proportion of the heritability for human
  height.
\newblock \emph{Nature Genetics}, 42\penalty0 (7):\penalty0 565--569, 2010.

\bibitem[Yang et~al.(2014)Yang, Zaitlen, Goddard, Visscher, and
  Price]{Yang:2014aa}
Jian Yang, Noah~A Zaitlen, Michael~E Goddard, Peter~M Visscher, and Alkes~L
  Price.
\newblock Advantages and pitfalls in the application of mixed-model association
  methods.
\newblock \emph{Nature Genetics}, 46\penalty0 (2):\penalty0 100--106, 2014.

\bibitem[Yuan and Lin(2006)]{Yuan06modelselection}
Ming Yuan and Yi~Lin.
\newblock Model selection and estimation in regression with grouped variables.
\newblock \emph{Journal of the Royal Statistical Society, Series B},
  68:\penalty0 49--67, 2006.

\bibitem[Zhang et~al.(2018)Zhang, Kuang, He, Idiga, Li, Chen, Yang, Cai, Zhang,
  and Potthoff]{Zhang:2018aa}
Peng Zhang, Henry Kuang, Yanlin He, Sharon~O Idiga, Siming Li, Zhimin Chen,
  Zhao Yang, Xing Cai, Kezhong Zhang, and Matthew~J Potthoff.
\newblock {NRG1-Fc} improves metabolic health via dual hepatic and central
  action.
\newblock \emph{JCI Insight}, 3\penalty0 (5):\penalty0 e98522, 2018.

\bibitem[Zhu and Stephens(2018)]{Zhu:2018aa}
Xiang Zhu and Matthew Stephens.
\newblock Large-scale genome-wide enrichment analyses identify new
  trait-associated genes and pathways across 31 human phenotypes.
\newblock \emph{Nature Communications}, 9\penalty0 (1):\penalty0 4361, 2018.

\end{thebibliography}


\clearpage

\begin{flushleft}
	{\Large{\textbf{Supplementary Material to ``Interpreting Deep Neural Networks Through Variable Importance''}}}
	\newline
	\\
	Jonathan Ish-Horowicz\textsuperscript{1}, Dana Udwin\textsuperscript{2}, Kayla Scharfstein\textsuperscript{3}, Seth Flaxman\textsuperscript{1}, Lorin Crawford\textsuperscript{2$\dagger$}, and Sarah Filippi\textsuperscript{1$\dagger$}
	\\
	\bigskip
	\bf{1} Department of Mathematics, Imperial College London, London SW7 2AZ, UK
	\\
	\bf{2} Department of Biostatistics, Brown University, Providence, RI, USA
	\\
	\bf{3} Division of Applied Mathematics, Brown University, Providence, RI, USA
	\\
	\bigskip
	$\dagger$ Corresponding E-mail: lorin\_crawford@brown.edu; s.filippi@imperial.ac.uk  
\end{flushleft}

\setcounter{figure}{0}
\setcounter{table}{0}
\setcounter{equation}{0}
\setcounter{section}{0}
\captionsetup[table]{name=Supplementary Table}
\captionsetup[figure]{name=Supplementary Figure}
\makeatletter




\section{Projection Operators in the Presence of Collinearity}\label{AppendixA}

In this section, our goal is to motivate the use of the covariance projection operator for the effect size analogue in Bayesian neural networks. We do this via a small simulation study which shows that the conventional linear estimation of regression coefficients is unstable in applications with highly collinear predictors. Here, we generate a synthetic design matrix with $n=5000$ individuals and $p=2$ covariates ($\x_1$ and $\x_2$) randomly drawn from standard normal distributions. We then assess two simulation scenarios with continuous outcomes created under the following linear model 
\begin{equation*}
\y = 2\x_1 - 2\x_2 +\bvarepsilon, \quad \quad \bbeta = [2,-2], \quad \quad \bvarepsilon\sim\cN(\bm{0},\bI).
\end{equation*}
In the first simulation scenario, $\x_1$ and $\x_2$ are uncorrelated; while, in the second scenario, the two covariates are set to share a Pearson correlation coefficient of $\rho = 0.999$. In each case, we compare the classic ordinary least squares (OLS) estimate for regression coefficients $\widehat{\bbeta} = (\X^{\T}\X)^{-1}\X^{\T}\y$ and the proposed covariance effect size analogue $\tbbeta = [\text{cov}(\x_1,\y),\text{cov}(\x_2,\y)]$. Figure \ref{fig:Simulation} depicts the results for both cases repeated 100 different times. In Supplementary Figure \ref{fig:Simulation}A, we see that both types of estimators are able to properly capture the true effects when the predictors are uncorrelated. This finding is expected. However, in the extremely collinear scenario with $\x_1\approx\x_2$, the total true effect size in the simulation is effectively equal to $\beta = 2-2 = 0$. The OLS estimators are unstable under this condition, while the covariance effect size analogues accurately and robustly estimate this value (see Supplementary Figure \ref{fig:Simulation}B).

\begin{figure}[H]
	\centering
	\includegraphics[width=\columnwidth]{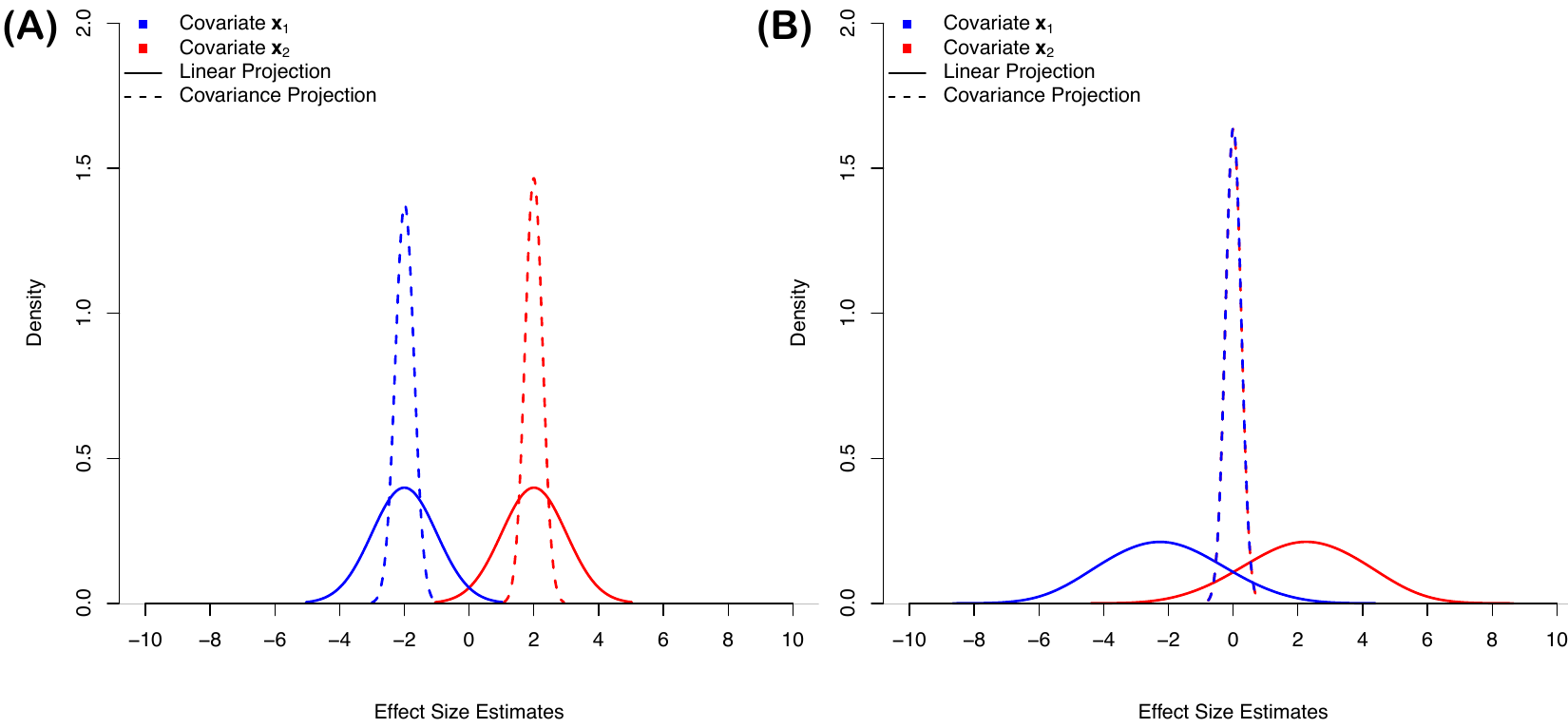}
	\caption{Results from a small simulation study showing the robustness of the covariance projection operator in the presence of collinear predictor variables. Synthetic data is generated as $\bf{y} = 2\bf{x}_1-2\bf{x}_2+\boldsymbol\varepsilon$ with $\boldsymbol\varepsilon\sim{\cal N}(\bm{0},\mathbf{I})$. OLS effect sizes are compared as a baseline. In panel (A), outcomes variables are generated with uncorrelated predictors; while in panel (B), the two covariates have a Pearson correlation coefficient of $\rho = 0.999$.}
	\label{fig:Simulation}
\end{figure}


\section{Covariance Projections and Marginal Association Tests}\label{AppendixB}

In this section, we prove a connection between the covariance projection operator and the conventional hypothesis testing strategies for marginal feature associations. Assume that we have an $n$-dimensional outcome variable $\by$ that is to be modeled by an $n\times p$ design matrix $\bX$. In linear regression, a simple (yet effective) approach is to take each covariate $\bx_j$ in turn and assess associations based upon a two-tailed alternative hypothesis. The significance of this test is then summarized via p-values (e.g.~$\widehat{p}_j$ for feature $j$), which may then be ranked in the order of importance from smallest to largest. Here, we show that the effect size analogues $\tbbeta$ correspond exactly to the test statistics for this frequented univariate approach. 

Begin by recalling that the covariance projection operator simply produces the sample covariance between a given predictor variable $\bx_j$ and the model predictions $\bm{f}$ --- where both largely positive or negative covariances are informative. Next, recall that the sample covariance between two random variables is equal to their Pearson correlation coefficient ($\rho$) multiplied by their respective standard errors $\sigma_X$ and $\sigma_Y$,
\begin{align}
\text{cov}(X,Y) = \rho\,\sigma_X\sigma_Y.
\end{align}
The standard formula for p-values starts by calculating a $t$-statistic of the following form
\begin{equation}
T_j = \rho_j \sqrt{\frac{n-2}{1-\rho_j^2}}, \quad \quad j = 1,\ldots,p.
\end{equation}
Corresponding p-values are then computed by comparing these values to a Student's $t$-distribution function under the null hypothesis --- with the intuition being that larger test statistics will result in smaller p-values. We now verify that these transformations are all monotonic --- thus, our proposed covariance effect size analogue will result in the same ranking of variable importance as the classical $t$-test.  

\begin{thm}{If two predictor variables have covariance effect size analogues such that $\tbeta_1 = \text{cov}(\x_1,\f) > \text{cov}(\x_2,\f)=\tbeta_2$, then the resulting p-values from a t-test with these features will have the relationship $\widehat{p}_1 < \widehat{p}_2$.}
\end{thm}
\begin{proof}
	Consider the covariance projection operation on two different predictor variables, $\text{cov}(\x_1,\f) > \text{cov}(\x_1,\f)$. Since standard deviations are positive
	\begin{align*}
	\text{cov}(\x_1,\f)\sigma_{\x_1}\sigma_{\f} > \text{cov}(\x_1,\f)\sigma_{\x_2}\sigma_{\f} \quad \quad \Longleftrightarrow \quad \quad \rho_1 > \rho_2.
	\end{align*}
	The same applies when multiplying both sides by $\sqrt{n-2}$. Also note that since we are concerned with the magnitude of covariances (and subsequently correlations), 
	\begin{align*}
	\rho_1 > \rho_2 \quad \Longleftrightarrow \quad \quad \sqrt{1-\rho_1} \le \sqrt{1-\rho_2}.
	\end{align*}
	Therefore we conclude that 
	\begin{align*}
	\rho_1 \sqrt{\frac{n-2}{1-\rho_1^2}} > \rho_2\sqrt{\frac{n-2}{1-\rho_2^2}} \quad \quad \Longleftrightarrow \quad \quad T_1 > T_2.
	\end{align*}
	Since the distribution function is monotonic, $\widehat{p}_1 < \widehat{p}_2$.
\end{proof}


\section{Training Procedure for Bayesian Neural Networks} \label{sec:bnn-archtecture-and-training}

We now detail the Bayesian neural network (BNN) architectures and training procedures used to derive the results presented in the main text: 
\begin{itemize}
	\item In the simulation study detailed in Section \ref{subsec:sim-studies}, the network consisted of two hidden, fully-connected layers with 32 and 16 units each and rectified linear unit (ReLU) activations. The output layer is specified as a hierarchical Bayesian model, as described in Section \ref{subsec:bnn-arch-for-rate}, with a sigmoid activation.
	\item For the MNIST study in Section \ref{subsec:results-mnist} and the Large Movie Review (IMDB) sentiment analyses, the network had a convolutional layer with 32 filters and stride 5, whose flattened output was passed to two fully-connected layers with 256 and 128 units and ReLU activation. The final layer was again specified with Bayesian hierarchical priors and a sigmoid activation.
	\item In the groupRATE application to a mice genome-wide association (GWA) study in Section \ref{sec:grouprate}, the network is made up of a series of three fully-connected layers with 512 units and ReLU activation. We applied dropout to each of these layers with rates of 0.5, 0.5, and 0.2, respectively. The output layer is again specified as in Section \ref{subsec:bnn-arch-for-rate} but with activation set to be the identity function, since this application focused on modeling continuous traits. 
\end{itemize}
In all four cases, the BNNs were trained for 50 epochs and implemented early stopping (with a patience of 2 epochs) based on the accuracy of \textit{(i)} a held-out validation set that contained 30\% of the training examples (for the simulations, MNIST, and IMBD analyses), or \textit{(ii)} the mean squared error on the training set. The latter is done for the mice genetic study since there was insufficient amount of data for a distinct validation set (total sample sizes $n\approx$ 2000). The Adam optimizer with a learning rate of $1\times 10^{-3}$ was used in all three cases \cite{kingma2014adam}.

\section{Cross-Validation of Mimic Models} \label{sec:cv-mimic-models}

For results in the main text, we used random search with 5-fold cross-validation for the random forest, gradient boosting machine, group lasso, and group random forest mimic models. Each random search fit 30 different models using \texttt{scikit-learn} \cite{scikit-learn}.





\clearpage 
\newpage

\section{Additional Supplementary Figures}

\begin{figure}[H]
	\centering
	\includegraphics[width=\textwidth]{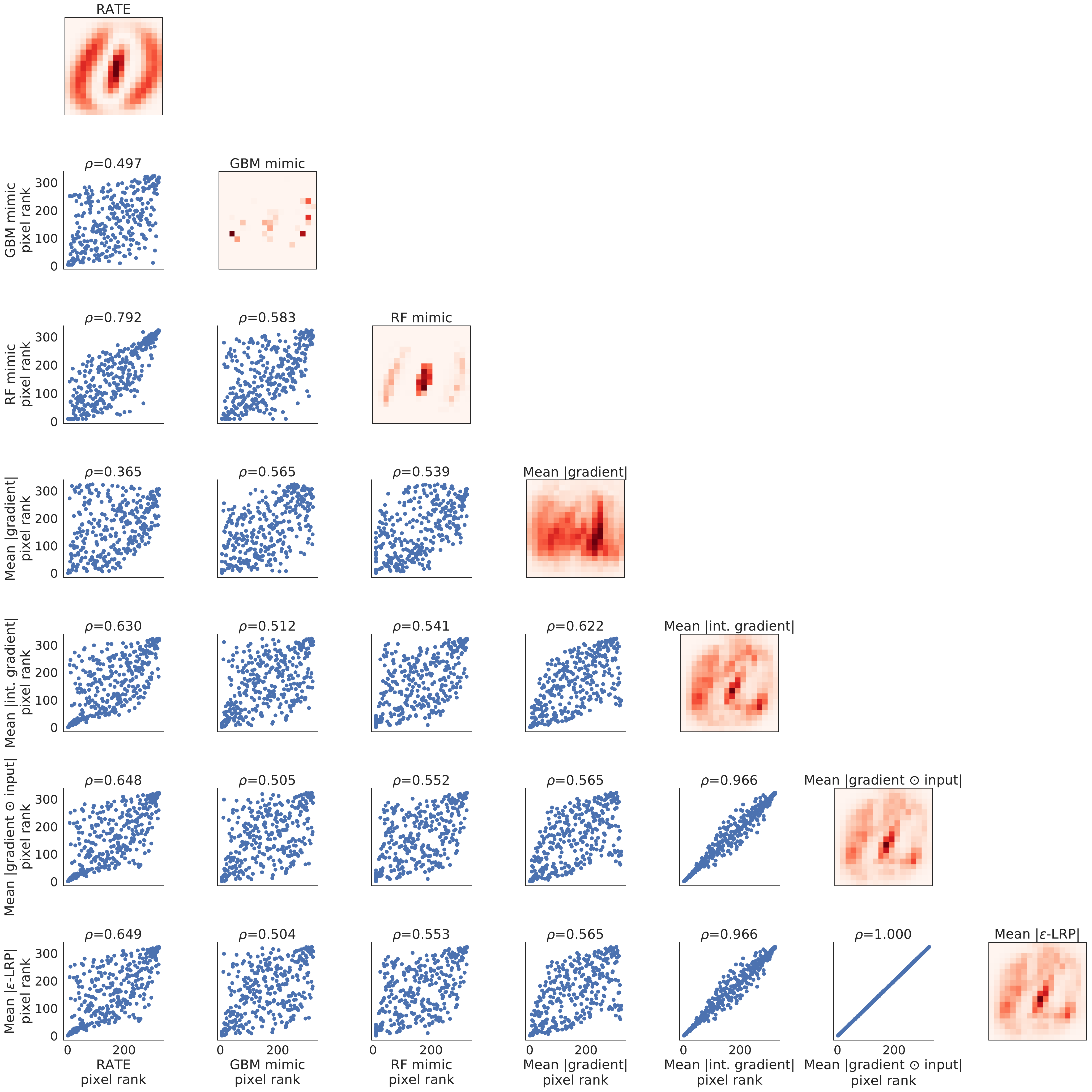}
	\caption{Pixel importance according to all the variable importance methods when classifying ones and zeros in MNIST. As the network only contained ReLU activations, gradient$\odot$input and $\varepsilon$-LRP are exactly equivalent \cite{ancona2017towards}.}
	\label{fig:mnist_zeroone_supplemental}
\end{figure}

\begin{figure}[H]
	\centering
	\includegraphics[width=\textwidth]{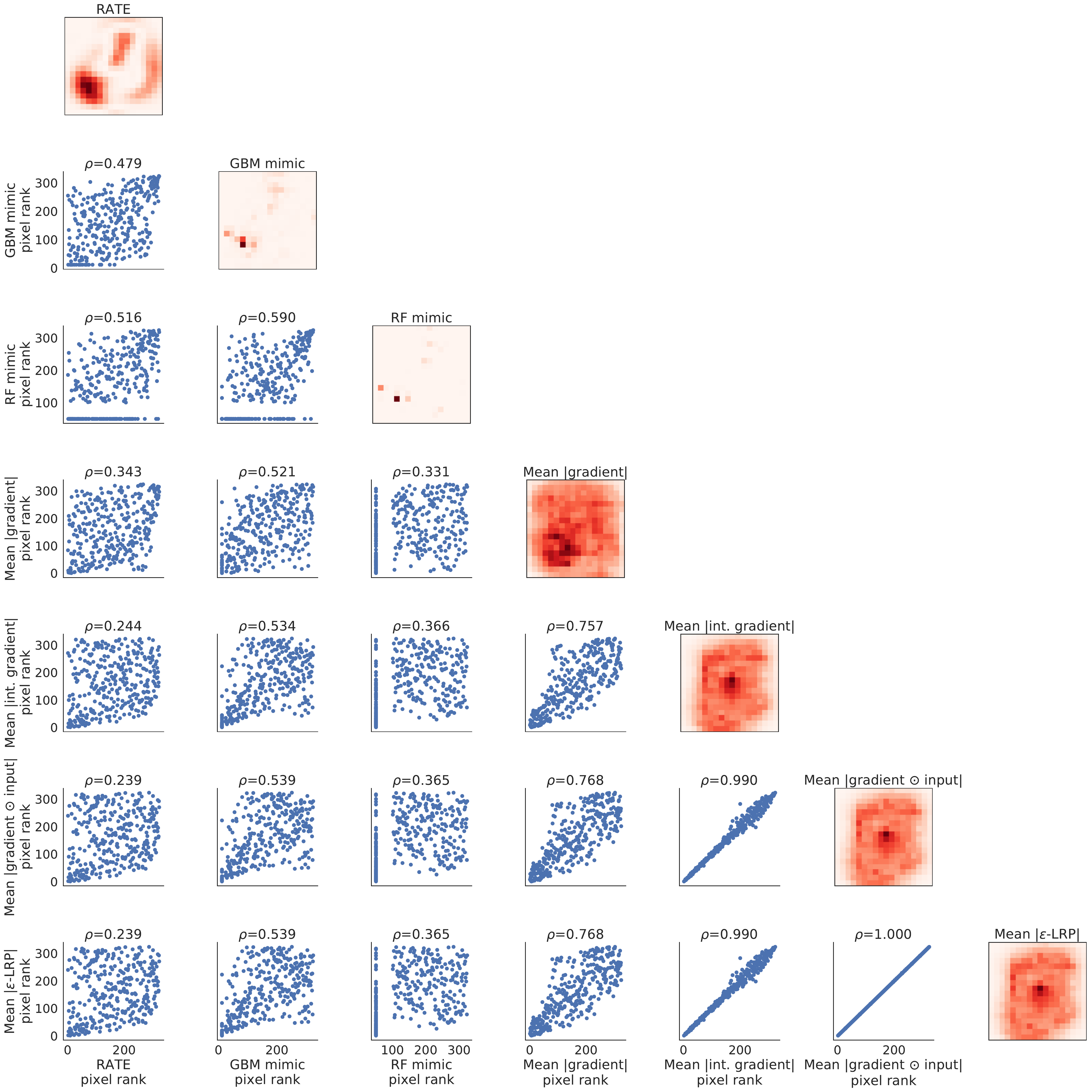}
	\caption{Pixel importance according to all the variable importance methods when classifying odds and evens in MNIST. As the network only contained ReLU activations, gradient$\odot$input and $\varepsilon$-LRP are exactly equivalent \cite{ancona2017towards}.}
	\label{fig:mnist_oddeven_supplemental}
\end{figure}

\begin{figure}[htb!]
	\centering
	\includegraphics[width=\columnwidth]{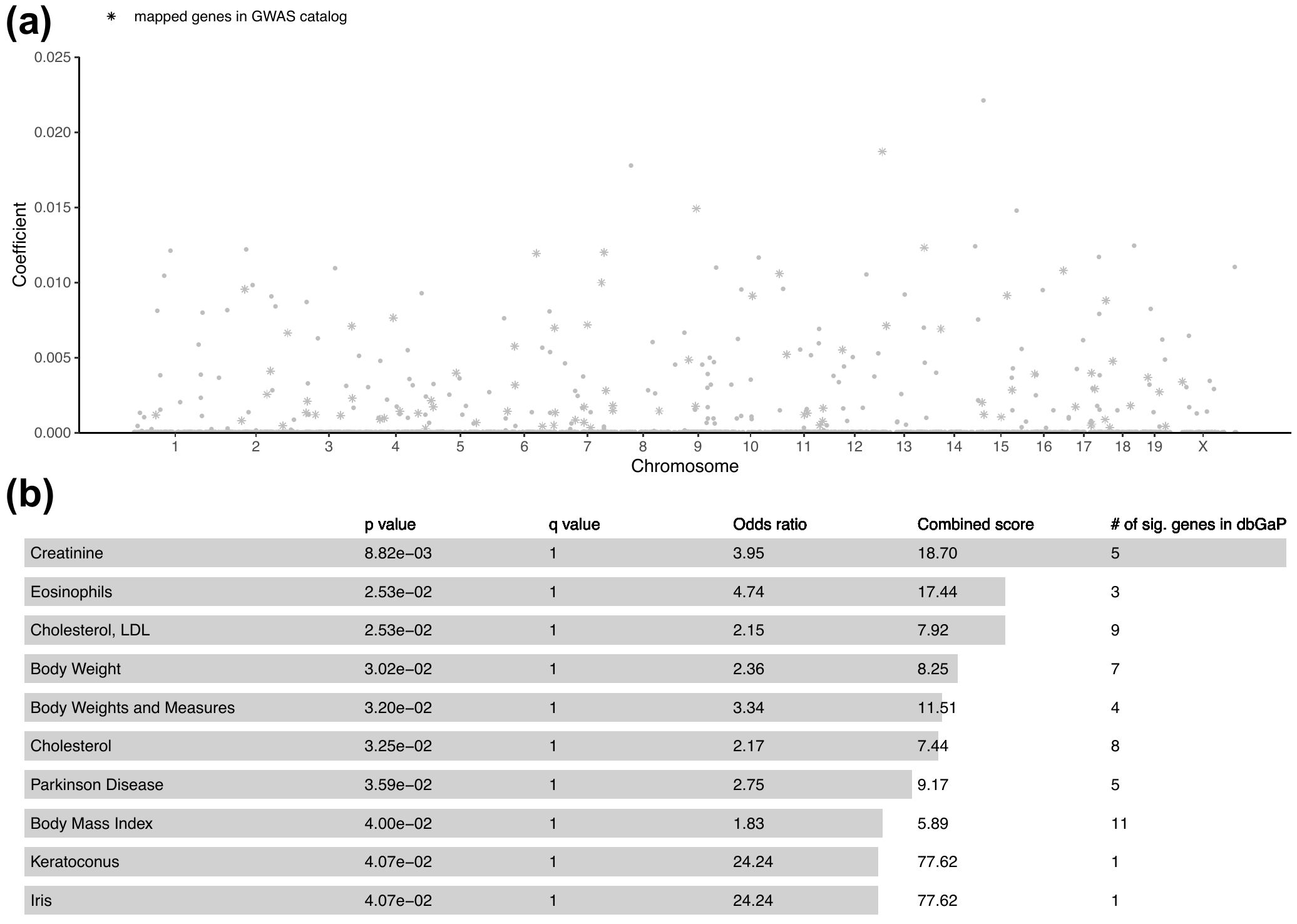}
	\caption{Gene-level association results from applying a group lasso mimic model to body mass index (BMI) in the heterogeneous stock of mice dataset. Panel (a) depicts the absolute value for each gene coefficient plotted against their genomic positions. We annotate significant genes identified by the group lasso (according to genome-wide threshold set to $T = $1/3749 genes = $2.67\times10^{-4}$) that overlap with those validated in the database of Genotypes and Phenotypes (dbGaP). In panel (b), we conduct gene set enrichment analysis using Enrichr \cite{Chen2013, 10.1093/nar/gkw377} to identify dbGaP categories enriched for significant gene-level associations reported by the group lasso. We highlight categories with $Q$-values (i.e., false discovery rates) less than 0.05 and annotate corresponding genes in the Manhattan plot.} 
	\label{Fig_S4}
\end{figure}

\begin{figure}[htb!]
	\centering
	\includegraphics[width=\columnwidth]{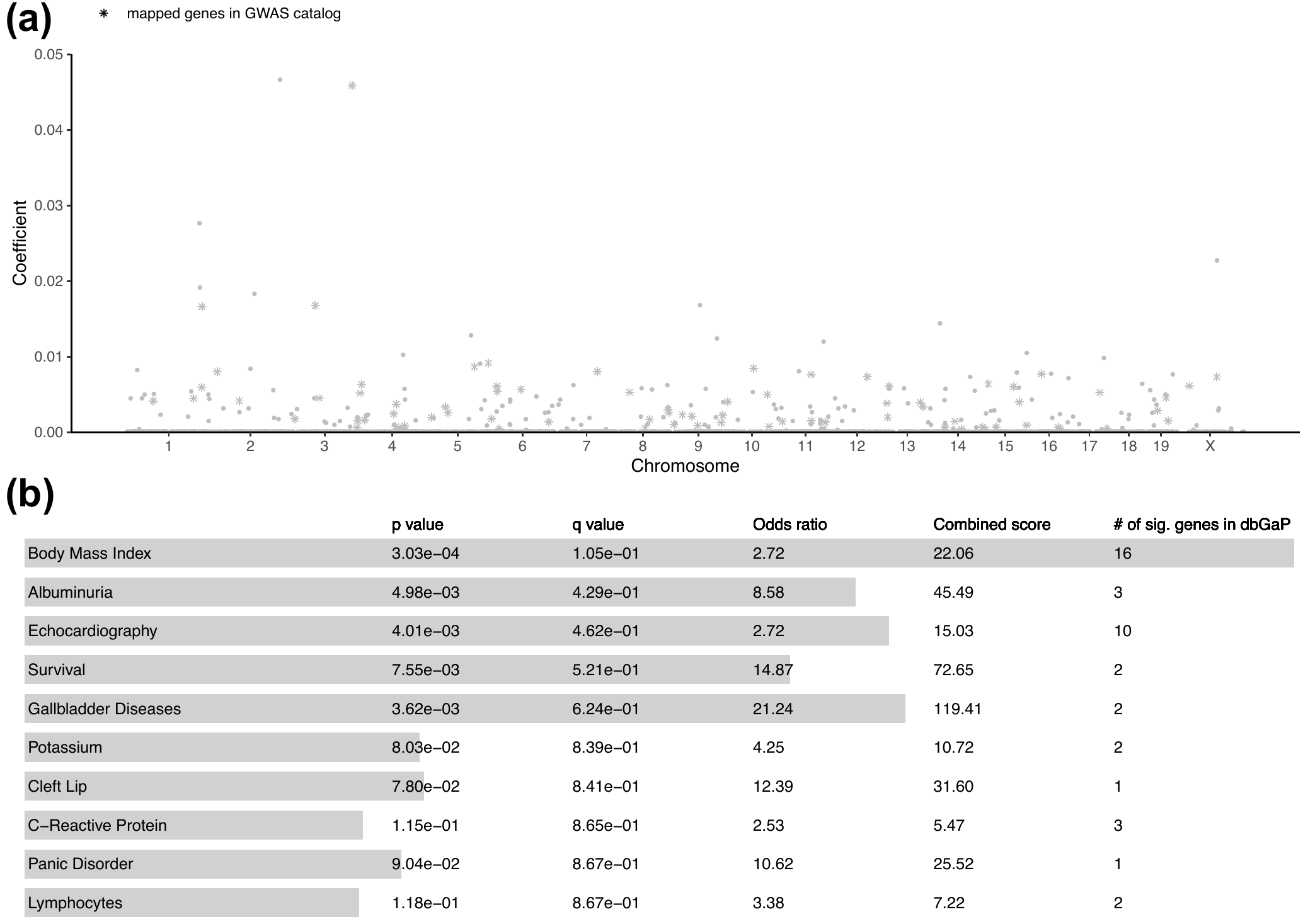}
	\caption{Gene-level association results from applying a group lasso mimic model to high-density lipoprotein (HDL) content in the heterogeneous stock of mice dataset. Panel (a) depicts the absolute value for each gene coefficient plotted against their genomic positions. We annotate significant genes identified by the group lasso (according to genome-wide threshold set to $T = $1/3749 genes = $2.67\times10^{-4}$) that overlap with those validated in the database of Genotypes and Phenotypes (dbGaP). In panel (b), we conduct gene set enrichment analysis using Enrichr \cite{Chen2013, 10.1093/nar/gkw377} to identify dbGaP categories enriched for significant gene-level associations reported by the group lasso. We highlight categories with $Q$-values (i.e., false discovery rates) less than 0.05 and annotate corresponding genes in the Manhattan plot.} 
	\label{Fig_S5}
\end{figure}

\begin{figure}[htb!]
	\centering
	\includegraphics[width=\columnwidth]{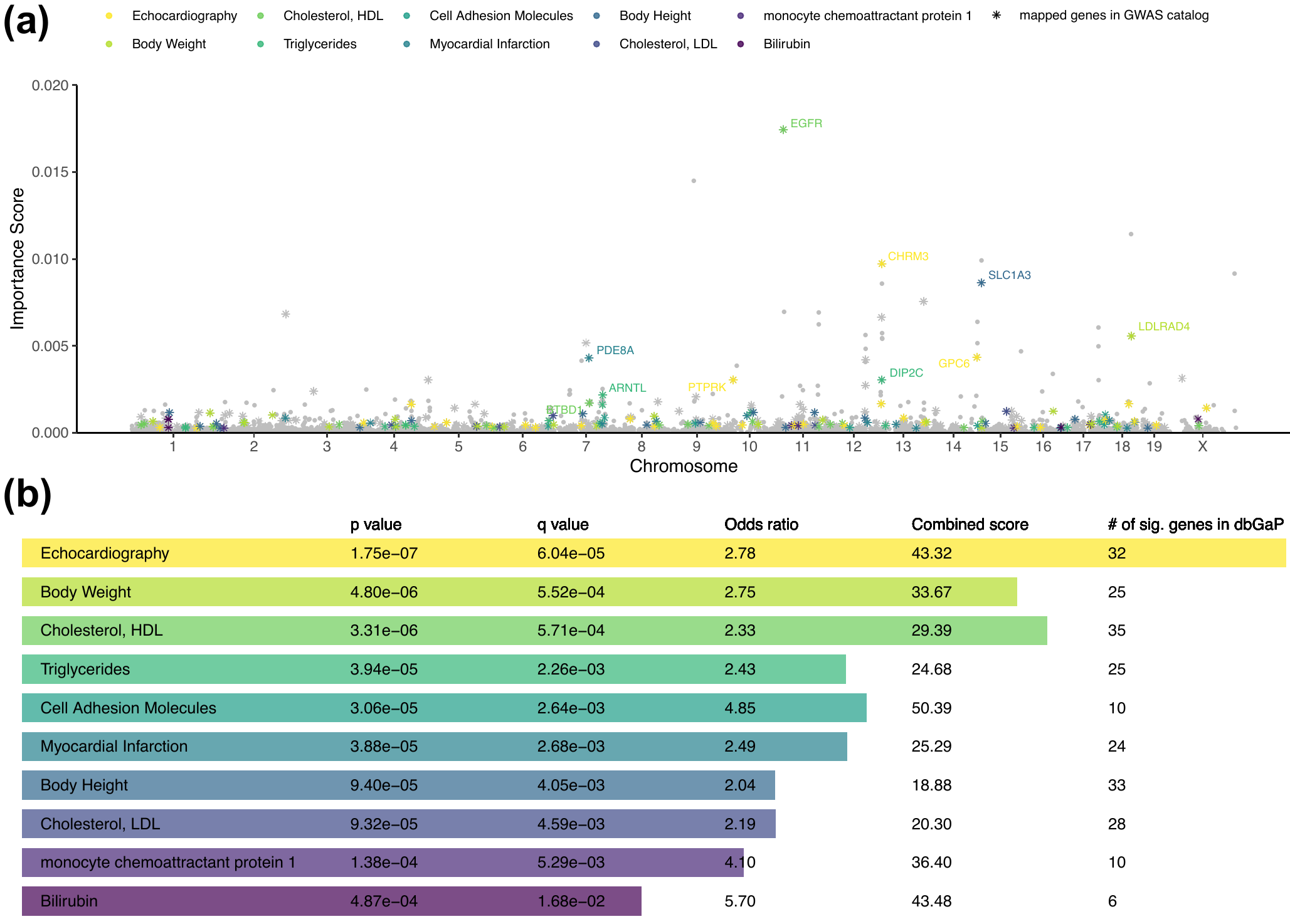}
	\caption{Gene-level association results from applying a group random forest mimic model to body mass index (BMI) in the heterogeneous stock of mice dataset. Panel (a) depicts the importance score for each gene plotted against their genomic positions. We annotate significant genes identified by the group random forest (according to genome-wide threshold set to $T = $1/3749 genes = $2.67\times10^{-4}$) that overlap with those validated in the database of Genotypes and Phenotypes (dbGaP). In panel (b), we conduct gene set enrichment analysis using Enrichr \cite{Chen2013, 10.1093/nar/gkw377} to identify dbGaP categories enriched for significant gene-level associations reported by the group random forest. We highlight categories with $Q$-values (i.e., false discovery rates) less than 0.05 and annotate corresponding genes in the Manhattan plot.} 
	\label{Fig_S6}
\end{figure}

\begin{figure}[htb!]
	\centering
	\includegraphics[width=\columnwidth]{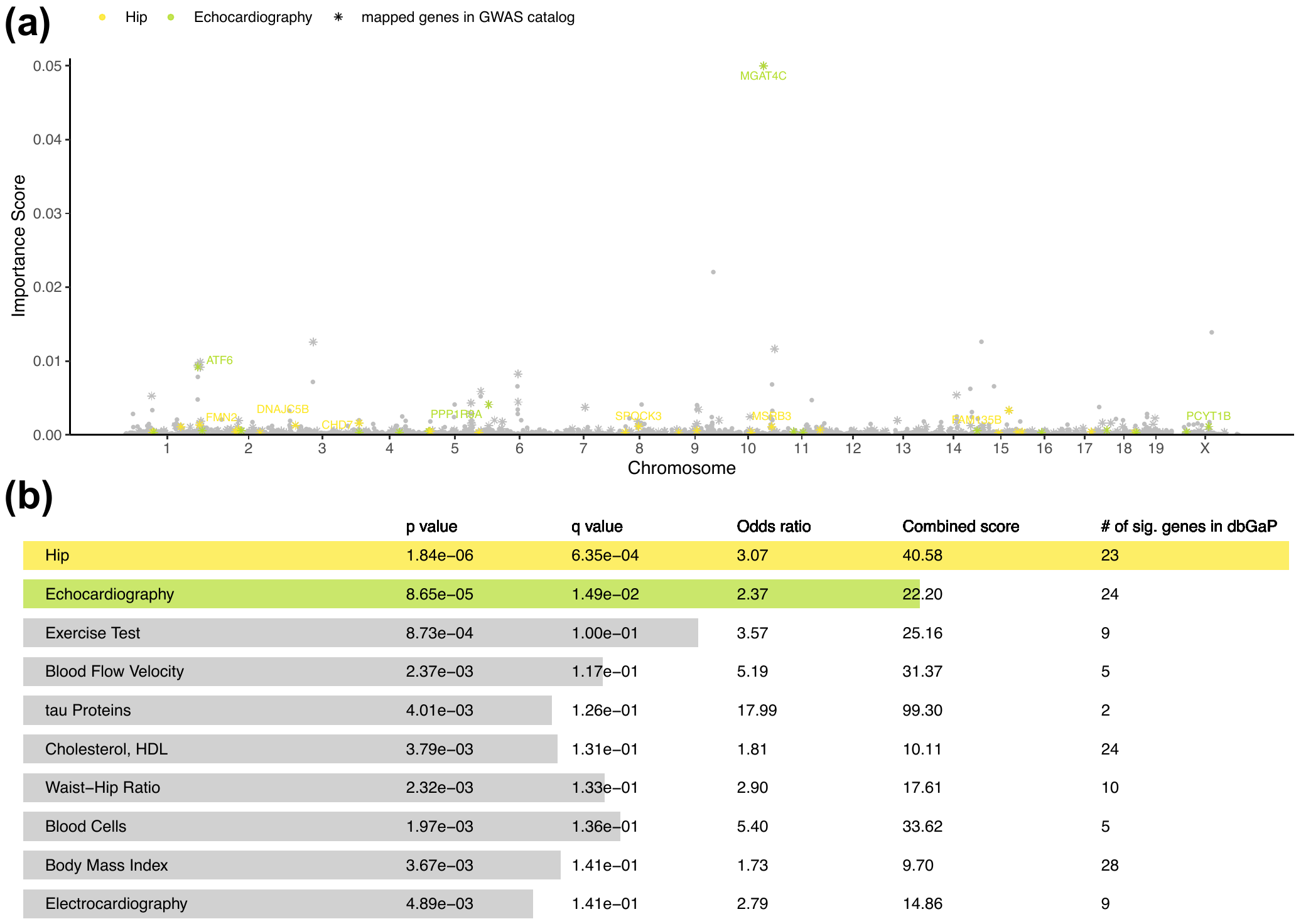}
	\caption{Gene-level association results from applying a group random forest mimic model to high-density lipoprotein (HDL) content in the heterogeneous stock of mice dataset. Panel (a) depicts the importance score for each gene plotted against their genomic positions. We annotate significant genes identified by the group random forest (according to genome-wide threshold set to $T = $1/3749 genes = $2.67\times10^{-4}$) that overlap with those validated in the database of Genotypes and Phenotypes (dbGaP). In panel (b), we conduct gene set enrichment analysis using Enrichr \cite{Chen2013, 10.1093/nar/gkw377} to identify dbGaP categories enriched for significant gene-level associations reported by the group random forest. We highlight categories with $Q$-values (i.e., false discovery rates) less than 0.05 and annotate corresponding genes in the Manhattan plot.} 
	\label{Fig_S7}
\end{figure}



\end{document}